\documentclass[11pt,a4paper]{article}

\usepackage{hyperref}
\usepackage{graphicx}
\usepackage{multirow}
\usepackage{amsmath,amssymb,mathtools}
\usepackage{dsfont}
\usepackage{amsthm}
\usepackage{mathrsfs}
\usepackage{xcolor}
\usepackage{manyfoot}
\usepackage{booktabs}
\usepackage{algorithm}
\usepackage{algpseudocode}
\usepackage{listings}
\usepackage{siunitx}
\usepackage{caption}
\usepackage{subcaption}
\usepackage{multirow,adjustbox}

\usepackage{natbib}
\usepackage{bbm}

\newtheorem{theorem}{Theorem}

\newtheorem{lemma}[theorem]{Lemma}

\newtheorem{remark}{Remark}

\newtheorem{definition}{Definition}

\newcommand{\DD}{{\mathcal D}}

\newcommand\ind{\protect\mathpalette{\protect\independenT}{\perp}}
\def\independenT#1#2{\mathrel{\rlap{$#1#2$}\mkern2mu{#1#2}}}

\numberwithin{equation}{section}

\def\argmin{\operatornamewithlimits{arg\,min}}
\def \indd{\mathds{1}}

\newtheorem{hyp}{Assumption}
\title{Transfer learning for causal forests}

\author{Bérénice-Alexia  Jocteur, Véronique Maume-Deschamps, \\
Pierre Ribereau\\
\footnotesize{Université Lyon 1, Centrale Lyon, INSA Lyon,}\\ \footnotesize{Université Jean Monnet,  CNRS,  ICJ UMR5208, 69622 Villeurbanne, France.}
}

\begin{document}

\maketitle

\begin{abstract}
Transfer learning addresses the challenge of transferring knowledge from one domain to another. Traditional transfer learning focuses on adapting models trained on a source domain (with many observations) to improve performance on a target domain (with few observations). In this work, we consider the case of a model shift and focus on transfer learning applied to a causal forest, namely HTERF. This causal forest aims to estimate the Conditional Average Treatment Effect (CATE).\\
The approach considered is the offset method presented by \cite{wang2016active}, adapted to a causal context. This method relies on the use of intermediate models to estimate the offset between source and target distributions. We establish an $L^1$-consistency result of our algorithm and derive a bound on the CATE estimation error of HTERF in the target domain, depending on the error of the intermediate models. Simulation studies demonstrate the good performance of this approach in different settings.
\end{abstract}

{\bf keywords: }{causal inference, causal forest, transfer learning, domain adaptation, offset method}
\section{Introduction}\label{sec:intro}

Estimating causal effects is essential for answering \textit{what-if} questions in many application fields, such as policy evaluation, medicine, or business strategy. The \textit{potential outcomes} framework~\cite{imbens2015causal} defines and estimates causal quantities such as the \textit{Average Treatment Effect} (ATE) and the \textit{Conditional Average Treatment Effect} (CATE). However, a major challenge arises when the target domain---where predictions are needed---has limited labeled data, while a related source domain has abundant observations. This scenario is common in real-world applications, where data distribution shifts, such as \textit{model shift} (i.e., \(P(Y^t(w)|\mathbf{X}^t) \neq P(Y^s(w)|\mathbf{X}^s)\) for \(w \in \{0,1\}\)), can undermine the performance of traditional causal inference methods.\\
\ \\
In this context, \textit{transfer learning} enables knowledge transfer from a data-rich source domain to a data-scarce target domain. Among transfer learning strategies, the \textit{offset method}~\cite{wang2016active}, originally designed for regression tasks, adjusts the source domain predictions using a correction learned from the target domain. This approach is particularly appealing in causal inference, where the goal is to estimate the CATE in the target domain while leveraging the structural information available in the source domain.

In this work, we focus on adapting the offset method to the causal setting, specifically for estimating the CATE under model shift. Our approach integrates the offset method with the \textit{Heterogeneous Treatment Effect-based Random Forest} (HTERF)~\cite{jocteur2024heterogeneous}, a specialized random forest algorithm designed for CATE estimation. HTERF improves upon existing methods, such as Generalized Random Forests~\cite{athey2019generalized}, by employing a splitting criterion tailored to capture treatment effect heterogeneity. We propose to combine the offset method with HTERF.

Our contributions are threefold:
\begin{itemize}
    \item We introduce a causal adaptation of the offset method, with two variants: one that separately models treated and control groups, and another that considers the treatment indicator as an additional covariate.
    \item We establish an \(L^1\)-consistency result for the CATE estimator in the target domain.
    \item We derive a generalization bound that highlights the role of intermediate model errors in the transfer process.
\end{itemize}
\ \\
The paper is organized as follows. Section \ref{background} contains the necessary background on causality, the HTERF algorithm, and domain adaptation. In Section \ref{sec:offset}, we detail the offset method. Section~\ref{sec:causal} presents our adaptation of the offset method to CATE estimation and our \(L^1\) consistency result, along with a generalization bound. A simulation study is presented in Section \ref{sec:simul}. A short discussion can be found in Section \ref{sec:conclusion}. Most technical proofs are postponed to the final Appendix.

\section{Some background}\label{background}
In this section, we provide a brief overview of the background necessary to understand the article. It includes summary elements on the causal framework, the HTERF algorithm, and domain adaptation.

\subsection{The causal framework}

Following the framework outlined in \cite{imbens2015causal}, the potential outcomes denoted \(Y(1)\) and \(Y(0)\) are defined as the outcomes that would have been observed if treatment or control had been assigned to the quantity of interest \(Y\), respectively. Let \(Y = Y(W)\) be the observed outcome, where \(W\) represents a binary treatment. Additionally, we incorporate a set of covariates \(\mathbf{X} \in \mathbb{R}^d\). The conditional average treatment effect (CATE) at \(\mathbf{x}\) is defined as follows:
\begin{equation}\label{eq:taux}
\tau(\mathbf{x}) = \mathbb{E}\left[ Y(1) - Y(0) |\mathbf{X} = \mathbf{x}\right].
\end{equation}
The average treatment effect (ATE) is:
\[
\mathbb{E}\left[ Y(1) - Y(0) \right].
\]
A standard assumption for the identifiability of CATE is unconfoundedness (\cite{rosenbaum1983central}), meaning that conditionally on \(\mathbf{X}\), the treatment assignment \(W\) is independent of the potential outcomes for \(Y\):
$$ \left\{ Y(1), Y(0)\right\} \ind W | \mathbf{X}.$$
Many algorithms in the literature allow for the evaluation of CATE: causal forests, meta-learners, and causal neural networks, among others. In what follows, we focus on HTERF, a short presentation of which is given in the next section.

\subsection{HTERF}
The transfer algorithms we present are based on HTERF, a special case of random forest introduced in \cite{jocteur2024heterogeneous}. It differs from the GRF model of \cite{athey2019generalized} by a new splitting criterion used to construct the trees. Given a sample \(\left(W_i, \mathbf{X}_i, Y_i\right)_{i=1,\dots,n} \in \{0,1\} \times \mathbb{R}^d \times \mathbb{R}\), it provides an estimator \(\hat{\tau}_{B,n}(\mathbf{x})\) of the CATE \(\tau(\mathbf{x})\). With some assumptions on the distribution of \((W, \mathbf{X}, Y)\) and on the construction of the forest, an almost sure convergence result of \(\hat{\tau}_{B,n}(\cdot)\) to \(\tau(\cdot)\) is obtained, as well as an interpretability result. This algorithm has been implemented in {\tt Julia}, in the package { \tt CausalForest}, \cite{jocteur2023julia}\footnote{\url{https://juliapackages.com/p/causalforest}}. For completeness, we summarize below the main points in HTERF construction. For further background knowledge on the HTERF algorithm and the associated theoretical results, we refer the readers to \cite{jocteur2024heterogeneous}. Let us remark that HTERF outperforms GRF and meta-learners on most of the tested settings.

The splitting criterion used in HTERF is designed to maximize the difference in treatment effect between child nodes. Specifically, Definition (\ref{eq:taux}) is used to define this splitting criterion. The tree-building process of HTERF relies on the optimization of the splitting criterion \(\Delta(A, j, z)\):

\begin{eqnarray*}
\Delta(A, j, z) &=& \frac{|A_L||A_R|}{|A|^2} \left( \left( \overline{Y}_{A_{L1}} - \overline{Y}_{A_{L0}} \right) - \right.\\
&&\left.\left( \overline{Y}_{A_{R1}} - \overline{Y}_{A_{R0}} \right) \right)^2,
\end{eqnarray*}
where $A$ is an hyper-rectangle of the form $\displaystyle\prod_{i=1}^d [a_i, b_i]$,
\(j\) belongs to a randomly chosen subset \(\mathcal{M}_{try}\) of \(\{1, \ldots, d\}\), \(z \in A^j = [a_j, b_j]\), \(\overline{Y}_A\) is the empirical mean of the \(Y_i\)'s for \(\mathbf{X}_i \in A\), and
\begin{align*}
    A_{L1} &= \{ \mathbf{X}_i \in A \mid {X}_i^{(j)} < z, W_i = 1 \}, \\
    A_{L0} &= \{ \mathbf{X}_i \in A \mid {X}_i^{(j)} < z, W_i = 0 \}, \\
    A_{R1} &= \{ \mathbf{X}_i \in A \mid {X}_i^{(j)} \geq z, W_i = 1 \}, \\
    A_{R0} &= \{ \mathbf{X}_i \in A \mid {X}_i^{(j)} \geq z, W_i = 0 \}, \\
    A_L &= A_{L1} \cup A_{L0}, \\
    A_R &= A_{R1} \cup A_{R0}.
\end{align*}
\ \\
The estimation of \(\tau\) is then given by:
\[
\hat{\tau}_{B,n}(\mathbf{x}) = \sum_{i : W_i = 1} \alpha_i(\mathbf{x}) Y_i - \sum_{i : W_i = 0} \alpha'_i(\mathbf{x}) Y_i,
\]
where \(\alpha\) (respectively, \(\alpha'\)) are the weights from the forest associated with observations where \(W_i = 1\) (respectively, \(W_i = 0\)).\\
\ \\
Let us now give a short overview of transfer learning, which is useful to understand the concepts, even if we shall adapt it to our causal framework context. 

\subsection{Transfer learning}

Transfer learning is a machine learning technique that leverages knowledge gained from solving one problem and applies it to a different but related problem. In traditional machine learning approaches, models are trained from scratch for each task, requiring substantial amounts of labeled data and computational resources. However, in real-world scenarios, labeled data might be scarce or expensive to acquire, hindering the effectiveness of such methods.\\
\ \\
Transfer learning addresses these limitations by transferring knowledge from a source domain where labeled data is abundant to a target domain where labeled data is scarce. This approach allows models to generalize better and achieve improved performance, particularly in situations where limited labeled data is available for training.\\
\ \\
Domain adaptation is a special case of transfer learning. In domain adaptation, the source and target domains share the same feature space (but have different distributions), while transfer learning includes cases where the target domain's feature space differs from the source feature space. In what follows, we consider the problem of supervised domain adaptation, where both source and target datasets are labeled.\\
\ \\
According to \cite{huyen2022designing}, in a supervised machine learning problem, the training dataset can be viewed as a set of samples from a joint distribution of \(P(\mathbf{X}, Y)\), where \(\mathbf{X}\) is the input and \(Y\) is the output. We are interested in modeling \(P(Y|\mathbf{X})\). Of course, \(P(\mathbf{X}, Y)\) can be decomposed as \(P(\mathbf{X}|Y) \times P(Y)\) or \(P(Y|\mathbf{X}) \times P(\mathbf{X})\).\\
\ \\
Different problems are addressed in transfer learning. The most common is the covariate shift, where the marginal distribution \(P(\mathbf{X})\) differs between source and target domains, but the conditional distribution \(P(Y|\mathbf{X})\) remains the same across the domains. Similarly, label shift can be defined as the case where \(P(Y)\) differs between source and target domains, but the conditional distribution \(P(\mathbf{X}|Y)\) stays the same across the domains. Finally, model shift or concept drift concerns cases where \(P(Y|\mathbf{X})\) changes, but \(P(\mathbf{X})\) remains the same.\\
\ \\
Different strategies are presented in \cite{huyen2022designing} to address these data distribution shifts. The first and simplest strategy is to train models on large and rich datasets, hoping that points following both source and target distributions will be present in this large dataset. This method requires access to large external datasets likely to contain both source and target distributions. Furthermore, it can be costly to train models on very large datasets. A second approach is to use algorithms dedicated to accounting for a certain type of shift; for example, the kernel mean matching (KMM) method (\cite{huang2006correcting}, \cite{gretton2006kernel}) allows dealing with covariate shift. \cite{zhang2013domain} proposes an approach to correct both covariate shift and label shift without using labels from the target distribution (unsupervised domain adaptation problem); similarly, \cite{zhao2019learning} proposed domain-invariant representation learning. \cite{wang2014active} introduces two methods to deal with covariate shift in real regression cases, using labeled source data. Finally, a third kind of approach to dealing with data distribution shift is to retrain the model with labeled target data, either by retraining the model from scratch with both source and target data or by resuming the training of the existing model (trained on the source) on the target data. This second option, named fine-tuning, is easily applicable to neural networks by using techniques such as freezing layers or warm starting.\\
\ \\
Transfer strategies can be extended to the causal context. We consider the source domain \((\mathbf{X}^s, Y^s, W^s)\) and the target domain \((\mathbf{X}^t, Y^t, W^t)\). We focus on the model shift case, i.e., \(P(Y^t(1)|\mathbf{X}^t) \neq P(Y^s(1)|\mathbf{X}^s)\) and \(P(Y^t(0)|\mathbf{X}^t) \neq P(Y^s(0)|\mathbf{X}^s)\). We assume that the distributions for \(\mathbf{X}^s\) and \(\mathbf{X}^t\) (respectively \(W^s\) and \(W^t\)) are the same. If this were not the case, the distributions of \(\mathbf{X}^s\) and \(\mathbf{X}^t\) could be matched by various methods dealing with covariate shift (e.g., KMM) without the use of \(Y\). The goal is then to estimate the CATE function on the target population. Recent work has been done to estimate ATE in a supervised domain adaptation setup in \cite{wei2024transfer}, where the nuisance parameters (such as the propensity score) are estimated using \(\ell^1\)-regularized transfer learning and then plugged into an ATE estimator. We can also mention \cite{kunzel2018transfer}, who proposed to transfer knowledge by using several strategies, such as using neural network (NN) weights estimated from the source domain as the warm start of the subsequent target domain NN training, or using NN weights estimated from the source domain and freezing some of its layers before backpropagating through the unfrozen ones when training on the target dataset. Neural networks with an architecture dedicated to causal transfer learning have also been proposed by \cite{bica2022transfer}. The method we propose is innovative, as it allows transfer learning on CATE estimation without using a neural network.

\section{The offset approach}\label{sec:offset}
We present the offset approach in a regression context. It is useful to understand the underlying concepts. 
\subsection{Presentation}
Let \(\mathcal{X} \in \mathbb{R}^d\) and \(\mathcal{Y} \in \mathbb{R}\) be the input and output spaces for a regression task for both source and target domains. Let \((\mathbf{Z}^s_i)_{i \in \{1, \dots, n\}} = ((\mathbf{X}^s_i, Y^s_i))_{i \in \{1, \dots, n\}}\) be the source dataset of size \(n\). We also consider: $(\mathbf{Z}^{tL}_i)_{i \in \{1, \dots, n_l\}} = $\\
$((\mathbf{X}^{tL}_i, Y^{tL}_i))_{i \in \{1, \dots, n_l\}}$ as the labeled target dataset of size \(n_l\). There is also an unlabeled target dataset on which we want to test the performance of transfer learning, denoted \((\mathbf{X}_i^{tU})_{i \in \{1, \dots, n_u\}}\) of size \(n_u\).

\begin{algorithm}
\caption{Offset algorithm}\label{alg:offset}
\begin{algorithmic}
\Require A source dataset \(\{\mathbf{X}^s_i, Y^s_i\}_{i \in \{1, \dots, n\}}\), a labeled target dataset \(\{\mathbf{X}_i^{tL}, Y_i^{tL}\}_{i \in \{1, \dots, n_l\}}\), and an unlabeled target dataset \(\{\mathbf{X}_i^{tU}\}_{i \in \{1, \dots, n_u\}}\).

\State Estimate a model \(\hat{f}^s\) that regresses \(\{Y^s_i\}\) against \(\{\mathbf{X}^s_i\}\).
\State Estimate a model \(\hat{f}^o\) that regresses \(\{\hat{Y}^o_i\} = \{Y^{tL}_i - \hat{f}^s(\mathbf{X}^{tL}_i)\}\) against \(\{\mathbf{X}^t_i\}\).
\State \(\{Y^{new}_i\} \gets \{Y^s_i + \hat{f}^o(\mathbf{X}^s_i)\}\)
\State Train a model \(M\) on \(\{\mathbf{X}^s_i, Y^{new}_i\} \cup \{\mathbf{X}^{tL}_i, Y^{tL}_i\}\).
\Ensure \(\{\hat{Y}^{tU}_i\} \gets \{M(\mathbf{X}^{tU}_i)\}\)
\end{algorithmic}
\end{algorithm}
\ \\
The offset algorithm (Algorithm \ref{alg:offset}) introduced by \cite{wang2014active} can be used with any regression machine learning algorithm for each estimator (namely \(\hat{f}^s, \hat{f}^o, M\)). In \cite{wang2015generalization}, a generalization bound is given when Kernel Ridge Regression (KRR) is used in the offset algorithm. For completeness, we provide below a short overview of KRR and the associated generalization bound.

\subsection{Using Kernel Ridge Regression}

KRR and the related notations are defined as follows.

\begin{definition}[\cite{bousquet2002stability}]
Let \\
\(\mathcal{S}_T = \{ \mathbf{Z}_1 = (\mathbf{X}_1, Y_1), \dots, \mathbf{Z}_n = (\mathbf{X}_n, Y_n) \}\)\\
be a training sample for a regression task in a reproducing kernel Hilbert space (see \cite{wahba2003introduction}) \(\mathcal{H}\) with associated norm \(\|.\|_{\mathcal{H}}\). Let \(\ell\) be the \(l^2\) loss function; then the KRR estimator is obtained by: given $\lambda >0$,
\[
\argmin_{h \in \mathcal{H}} \frac{1}{n} \sum_{i=1}^n \ell(h, \mathbf{Z}_i) + \lambda \|h\|_{\mathcal{H}}^2
\]
Two errors are defined:
\begin{itemize}
    \item \(R = \mathbb{E}[\ell(\mathcal{S}_T, \mathbf{Z})]\), the generalization error,
    \item \(R_{emp} = \frac{1}{n} \sum_{i=1}^n \ell(\mathcal{S}_T, \mathbf{Z}_i)\), the empirical error.
\end{itemize}
\end{definition}
\ \\
The following theorem gives a bound on the generalization error when KRR is used in the offset method.

\begin{theorem}[\cite{wang2015generalization}]
If KRR is used to estimate the three functions in the offset method, let \(R^t\) be the generalization error on the target dataset of the final model \(M\), \(R^s_{emps}\) the empirical error on the source model, and \(\Bar{R}^o_{emp}\) the empirical error of the estimator \(\hat{f}^o\) against \(\{\mathbf{X}^{tL}, \hat{Y}^o\}\). Then,
\[
R^t - 2\left(R^s_{emps} - \Bar{R}^o_{emp}\right) = O\left(\frac{1}{\sqrt{\lambda_o n_l}}\right),
\]
where \(\lambda_o\) is the hyperparameter of KRR.
\end{theorem}
\ \\
This result relies on Theorem 12 in \cite{bousquet2002stability}, which gives a property of uniform stability for the KRR algorithm. However, this property is not known for many other algorithms besides KRR (or only a weaker version of stability is obtained), which makes extensions of this result to the causal case presented in the following section challenging.

\section{Causal adaptation}\label{sec:causal}

We propose an offset method adapted to the causal framework.

\subsection{Overview}

Two causal adaptations of the offset method are proposed. In Algorithm \ref{alg:offsetcausal}, the treated and control populations are processed separately to estimate source and offset functions (one estimator for each group). In Algorithm \ref{alg:offsetcausalbis}, the treatment variable is considered as an additional covariate for the source and offset functions.\\
\ \\
The algorithms require a source dataset \(\DD^s = \{W^s_i, \mathbf{X}^s_i, Y^s_i\}_{i \in \{1, \dots, n\}}\) and a target dataset \(\DD^t = \{W^{tL}_i, \mathbf{X}^{tL}_i, Y^{tL}_i\}_{i \in \{1, \dots, n_l\}}\).

\begin{algorithm}[H]
\caption{Offset causal algorithm: separate models}\label{alg:offsetcausal}
\begin{algorithmic}
\Require A source dataset \(\DD^s\), a target dataset \(\DD^t\), and an unlabeled target dataset \(\{\mathbf{X}^{tU}_i\}_{i \in \{1, \dots, n_u\}}\).

\State Estimate a model \(\hat{f}_0^s\) that regresses \(\{Y^s_i\}_{W^s_i=0}\) against \(\{\mathbf{X}^s_i\}_{W^s_i=0}\) and a model \(\hat{f}_1^s\) that regresses \(\{Y^s_i\}_{W^s_i=1}\) against \(\{\mathbf{X}^s_i\}_{W^s_i=1}\).
\State Estimate a model \(\hat{f}^o_0\) that regresses \(\{Y^{tL}_i - \hat{f}^s_0(\mathbf{X}^{tL}_i)\}_{W_i^{tL}=0}\) against \(\{\mathbf{X}^{tL}_i\}_{W^{tL}_i=0}\) and a model \(\hat{f}^o_1\) that regresses \(\{Y^{tL}_i - \hat{f}^s_1(\mathbf{X}^{tL}_i)\}_{W^{tL}_i=1}\) against \(\{\mathbf{X}^{tL}_i\}_{W^{tL}_i=1}\).
\State   \(\{Y^{new}_i\} \gets \{Y^s_i + \hat{f}^o_{W^s_i}(\mathbf{X}^s_i)\}\)
\State Train an HTERF model \(M\) on \(\{W^s_i, \mathbf{X}^s_i, Y^{new}_i\} \cup \{W^{tL}_i, \mathbf{X}^{tL}_i, Y^{tL}_i\}\).
\Ensure \(\{\hat{\tau}^t(\mathbf{X}^{tU}_i)\} \gets \{M(\mathbf{X}^{tU}_i)\}\)
\end{algorithmic}
\end{algorithm}

\begin{algorithm}
\caption{Offset causal algorithm: unique model}\label{alg:offsetcausalbis}
\begin{algorithmic}
\Require A source dataset \(\DD^s\), a target dataset \(\DD^t\), and an unlabeled target dataset \(\{\mathbf{X}^{tU}_i\}_{i \in \{1, \dots, n_u\}}\).

\State Estimate a model \(\hat{f}^s\) that regresses \(\{Y^s_i\}\) against \(\{\mathbf{X}^s_i, W^s_i\}\).
\State Estimate a model \(\hat{f}^o\) that regresses \(\{Y^{tL}_i - \hat{f}^s(\mathbf{X}^{tL}_i, W^{tL}_i)\}\) against \(\{\mathbf{X}^{tL}_i, W^{tL}_i\}\).
\State \(\{Y^{new}_i\} \gets \{Y^s_i + \hat{f}^o(\mathbf{X}^s_i, W^s_i)\}\)
\State Train an HTERF model \(M\) on \(\{W^s_i, \mathbf{X}^s_i, Y^{new}_i\} \cup \{W^{tL}_i, \mathbf{X}^{tL}_i, Y^{tL}_i\}\).
\Ensure \(\{\hat{\tau}^t(\mathbf{X}^{tU}_i)\} \gets \{M(\mathbf{X}^{tU}_i)\}\)
\end{algorithmic}
\end{algorithm}
\ \\
We call \(\hat{f}^s_0\), \(\hat{f}^s_1\), \(\hat{f}^s\) the source estimators and \(\hat{f}^o_0\), \(\hat{f}^o_1\), \(\hat{f}^o\) the offset estimators. Any regression algorithm could be used to obtain the source and offset estimators. In practice, we obtained good results by using regression random forests.
\subsection{Convergence result}
As already mentioned, the unconfoundedness hypothesis is necessary in the potential outcomes setting. Recall that we consider the \textit{model shift} setting, i.e., \(\mathbf{X}^s \stackrel{{\mathcal{L}}}{=} \mathbf{X}^t\), \(W^s \stackrel{{\mathcal{L}}}{=} W^t\), and the conditional laws \(Y^s(W^s)|\mathbf{X}^s\) and \(Y^t(W^t)|\mathbf{X}^t\) differ. The unconfoundedness assumption has to be done on $(Y^s,\mathbf{X}^s, W^s)$ and $(Y^t,\mathbf{X}^t, W^t)$, which are also assumed to be independent (see Assumption \ref{hyp:vartrans}).   \\
Since the size of the source dataset is assumed to be large compared to the target dataset, we state a convergence theorem and a generalization bound on the causal offset algorithm if the HTERF model in the last step is only fit on the set \(\{W_i^s, \mathbf{X}_i^s, Y_i^{new}; i=1, \ldots, n\} = \mathcal{D}_n\).\\
\ \\
We use the following notations, as in \cite{jocteur2024heterogeneous}:
\begin{itemize}
\item \(\Theta_{\ell}, \ell=1, \ldots, B\) are independent random vectors, distributed as a generic random vector \(\Theta \) and independent of \(\mathcal{D}_n\). 
It contains indices of observations that are used to build each tree; indices of observations that are used for estimations in each tree;  indices of splitting candidate variables in each node.
\item \(\mathcal{D}_{n,1}^{\star}(\Theta_\ell)\) and \(\mathcal{D}_{n,2}^{\star}(\Theta_\ell)\) are the disjoint subsamples selected prior to tree construction; the first is used to build the tree, and the second allows the building of weights used during the estimation step.
\item \(A_n(\mathbf{x}; \Theta_\ell, \mathcal{D}_n)\) is the tree cell (subspace of \(\mathcal{X}\)) containing \(\mathbf{x}\).
\item \(N_{n,1}(\mathbf{x}; \Theta_\ell, \mathcal{D}_n)\) (resp. \(N_{n,0}(\mathbf{x}; \Theta_\ell, \mathcal{D}_n)\)) is the number of elements of \(\mathcal{D}_{n,2}^{\star}(\Theta_\ell)\) that fall into \(A_n(\mathbf{x}; \Theta_\ell, \mathcal{D}_n)\), such that \(W_i=1\) (resp. \(W_i=0\)).
\end{itemize}
Let us consider:
\(\tau^t_1(\mathbf{x}) = \mathbb{E}[Y^t(1)|\mathbf{X}^t = \mathbf{x}]\),
\(\tau^t_0(\mathbf{x}) = \mathbb{E}[Y^t(0)|\mathbf{X}^t = \mathbf{x}]\),\\
\(\hat{\tau}^{new}_1(\mathbf{x}) = \sum_{i:W^s_i=1} \alpha_i(\mathbf{x}) Y^{new}_i\), 
and \(\hat{\tau}^{new}_0(\mathbf{x}) = \sum_{i:W^s_i=0} \alpha'_i(\mathbf{x}) Y^{new}_i\), where
\begin{equation}\label{alpha}
\alpha_i(\mathbf{x}) = \frac{1}{B} \sum_{\ell=1}^B \frac{\indd_{\mathbf{X}^s_i \in A_n(\mathbf{x}; \Theta_\ell, \mathcal{D}_n) \land W_i^s=1 \land i \in \mathcal{D}_{n,2}^{\star}(\Theta_\ell)}}{N_{n,1}(\mathbf{x}; \Theta_\ell, \mathcal{D}_n)},
\end{equation}
\begin{equation}\label{alpha'}
\alpha'_i(\mathbf{x}) = \frac{1}{B} \sum_{\ell=1}^B \frac{\indd_{\mathbf{X}^s_i \in A_n(\mathbf{x}; \Theta_\ell, \mathcal{D}_n) \land W_i^s=0 \land i \in \mathcal{D}_{n,2}^{\star}(\Theta_\ell)}}{N_{n,0}(\mathbf{x}; \Theta_\ell, \mathcal{D}_n)}.
\end{equation}
Our aim is to estimate the target CATE:
\[\tau^t(\mathbf{x}) = \tau^t_1(\mathbf{x}) - \tau^t_0(\mathbf{x})\/.\]
We shall make the following assumptions on the distributions and models. There are slight differences depending on whether Algorithm \ref{alg:offsetcausal} (two offset models) or Algorithm \ref{alg:offsetcausalbis} (one single offset model) is concerned.
\begin{hyp}\label{hyp:vartrans}\ \\
\vspace*{-.5cm}\begin{itemize}
\item $(\mathbf{X}^s\/,W^s\/,Y^s)$ and $(\mathbf{X}^t\/,W^t\/,Y^t)$ are independent; the samples $\mathcal{D}^s$, $\mathcal{D}^t$ are independent.
\item the unconfoundedness property is assumed for both  $(\mathbf{X}^s\/,W^s\/,Y^s)$ and $(\mathbf{X}^t\/,W^t\/,Y^t)$:
$$\left\{ Y^s(1), Y^s(0)\right\} \ind W^s | \mathbf{X}^s\/, \  \left\{ Y^t(1), Y^t(0)\right\} \ind W^t | \mathbf{X}^t\/.$$
    \item \(\mathbf{X}^s\) and \(\mathbf{X}^t\) are distributed as \(\mathbf{X} = (X^{(1)}, \ldots, X^{(d)})\), which is a continuous random vector with independent coordinates. The density of \(\mathbf{X}\) is positive and bounded from above and below by positive constants.
    \item \(\mathbf{X}\) takes its values in \(\mathcal{X}\), which is assumed to be a positive compact hyper-rectangle of \(\mathbb{R}^d\): \\
    \(\mathcal{X} = \displaystyle\prod_{i=1}^d [u_i, v_i]\), \(-\infty < u_i \leq v_i < \infty\).
    \item \(W^s\) and \(W^t\) are distributed as \(W\), a binary variable.
    \item {\bf Case Algorithm \ref{alg:offsetcausal}} \\
  If $W^s=1$,   \(Y^s(1) = f_1^s(\mathbf{X}^s) + \varepsilon^s_{1}\), \(\varepsilon^s_1 \ind \mathbf{X}^s\), \\
    and if $W^t=1$,  \(Y^t(1) - f_1^s(\mathbf{X}^t) = f_1^o(\mathbf{X}^t) + \varepsilon^t_{1}\), \(\varepsilon^t_1 \ind  \mathbf{X}^t\). \\
    \(\varepsilon_1^s\) and \(\varepsilon_1^t\) are continuous centered random variables. \\
 If $W^s=0$,     \(Y^s(0) = f_0^s(\mathbf{X}^s) + \varepsilon^s_{0}\), \(\varepsilon^s_0 \ind \mathbf{X}^s\), \\
    and  if $W^t=0$, \(Y^t(0) - f_0^s(\mathbf{X}^t) = f_0^o(\mathbf{X}^t) + \varepsilon^t_{0}\), \(\varepsilon^t_0 \ind \mathbf{X}^t\). \\
    \(\varepsilon_0^s\) and \(\varepsilon_0^t\) are continuous centered random variables.
    \item {\bf Case Algorithm \ref{alg:offsetcausalbis}} \\
    \(Y^s(W^s) = f^s(\mathbf{X}^s, W^s) + \varepsilon^s\) \\
    and \(Y^t(W^t) - f^s(\mathbf{X}^t, W^t) = f^o(\mathbf{X}^t, W^t) + \varepsilon^t\). \\
    \(\varepsilon^s\) and \(\varepsilon^t\) are continuous centered random variables, independent respectively of $\mathbf{X}^s\/, W^s$ and $\mathbf{X}^t\/, W^t$.
    \item \(\mathbf{x} \mapsto f^s(\mathbf{x}, W)\), \(f^s_0(\mathbf{x})\), \(f^s_1(\mathbf{x})\), \(f^t(\mathbf{x}, W)\), \(f^t_1(\mathbf{x})\), \(f^t_0(\mathbf{x})\) are continuous. So in particular \(\mathbf{x} \mapsto \tau^t(\mathbf{x})\), \(\tau^t_1(\mathbf{x})\), \(\tau^t_0(\mathbf{x})\) and \(\mathbf{x} \mapsto \tau^s(\mathbf{x})\), \(\tau^s_1(\mathbf{x})\), \(\tau^s_0(\mathbf{x})\) are continuous.
\end{itemize}
\end{hyp}

\begin{hyp}\label{hyp:numbertrans}
The following assumptions are made on \(B\) (number of trees in HTERF), \(N_{n,1}(\mathbf{x}; \Theta, \mathcal{D}_n)\) resp. \(N_{n,0}(\mathbf{x}; \Theta, \mathcal{D}_n)\) (number of observations in a leaf node such as \(W=1\), resp. \(W=0\)), and on the construction of the trees:
\begin{enumerate}
\item \(B = \mathcal{O}(\sqrt{n})\), \(\exists C > 0\) and  \(\exists \beta > 1\)
such that \(B > C \frac{\sqrt{n}}{(\ln(n))^\beta}\), 
\item $\forall \mathbf{x} \in \mathcal{X}$,\\
$ \mathbb{E}[N_{n,1}(\mathbf{x}; \Theta, \mathcal{D}_n)] = \Omega(\sqrt{n}(\ln(n))^\beta)$.
\item $\forall \mathbf{x} \in \mathcal{X}$,\\
$\mathbb{E}[N_{n,0}(\mathbf{x}; \Theta, \mathcal{D}_n)] = \Omega(\sqrt{n}(\ln(n))^\beta)$.
\item \(\max_{\mathbf{x}, \Theta} N_{n,1}(\mathbf{x}; \Theta, \mathcal{D}_n) = o(n)\).
\item \(\max_{\mathbf{x}, \Theta} N_{n,0}(\mathbf{x}; \Theta, \mathcal{D}_n) = o(n)\).
\item At every step of the tree-building procedure, the probability that the next split is done along the \(j\)-th feature is bounded below by \(\pi/d\) for some \(0 < \pi \leq 1\) for all \(j = 1, \dots, d\).
\item HTERF, as described in \cite{jocteur2024heterogeneous}, uses an honest framework: the training sample is split into two parts: one is used to construct the splits of the trees, the other one is used to calculate the weights \(\alpha\) and \(\alpha'\). This second subsample verifies that each split leaves at least a fraction \(\delta\) of the available training sample such that \(W=1\) (resp. \(W=0\)) on each side of the split, for some \(0 < \delta \leq 0.5\).
\end{enumerate}
\end{hyp}
\noindent Let \(\mathcal{D}^s = (W_i^s, \mathbf{X}_i^s, Y_i^s; i=1, \ldots, n)\) be the source sample and \(\mathcal{D}^t = (W_i^t, \mathbf{X}_i^t, Y_i^t; i=1, \ldots, n_l)\) be the target sample; \(\mathcal{D} = \mathcal{D}^s \cup \mathcal{D}^t\).\\
\ \\
Let us first consider the case of Algorithm \ref{alg:offsetcausal}. The function \(\hat{f}^s_1\) is an estimator of \(f^s_1\) in the first step: if $W^s=1$,
\[
Y^s(1) = \hat{f}^s_1(\mathbf{X}^s) + \varepsilon_{1}^s + E_{1}^s(\mathcal{D}^s, \mathbf{X}^s)
\]
The function \(\hat{f}^o_1\) is an estimator of \(f^o_1\) in the second step: if $W^t=1$
\[
Y^t(1) - \hat{f}^s_1(\mathbf{X}^t) - E_1^s(\mathcal{D}^s, \mathbf{X}^t) = \hat{f}^o_1(\mathbf{X}^t) + E_1^o(\mathcal{D}, \mathbf{X}^t) + \varepsilon_{1}^t
\]
Finally, the third step leads to: if $W^s=1$
\begin{eqnarray*}
Y^{new}(1) &=& Y^s(1) + \hat{f}_1^o(\mathbf{X}^s) \\
&=& f_1^s(\mathbf{X}^s) +f_1^o(\mathbf{X}^s)  + \varepsilon_{1}^s +E_1^o(\mathcal{D}, \mathbf{X}^s)\/.
\end{eqnarray*}
In a similar fashion, we have: if $W^s=0$
 \[
 Y^{new}(0) = f_0^s(\mathbf{X}^s) +f_0^o(\mathbf{X}^s)  + \varepsilon_{0}^s +E_0^o(\mathcal{D}, \mathbf{X}^s) \/.
 \]
In the case of Algorithm \ref{alg:offsetcausalbis}, we proceed the same way and are led to: 
$$Y^{new}(W^s) = f^s(\mathbf{X}^s,W^s) +f^o(\mathbf{X}^s,W^s)  + \varepsilon^s +E^o(\mathcal{D}, \mathbf{X}^s,W^s) .
$$
It should be noted that, in case of Algorithm 2, $\tau_1^t (x)= f_1^s(x)+f_1^o(x)$ and in case of Algorithm 3, $\tau_1^t(x) = f^s(x,1)+f^o(x,1)$ (with similar expressions for $\tau_0^t$). As a consequence, the above expressions of $Y^{new}$ show that  it can be usefully used for the estimation of $\tau_1^t$, $\tau_0^t$ and thus $\tau^t$. To this aim,  HTERF is trained on:\\
\(\mathcal{D}_n = \{W_i^s, \mathbf{X}_i^s, Y_i^{new}\/, i=1\/,\ldots\/,n\}\), which gives the following estimator of the target CATE:\\
\(\hat{\tau}^{new}_{B,n}(\mathbf{X}) = \hat{\tau}^{new}_1(\mathbf{x}) - \hat{\tau}^{new}_0(\mathbf{x})\), \\
where \(\hat{\tau}^{new}_1(\mathbf{x}) = \sum_{i:W_i^s=1} \alpha_i(\mathbf{x}) Y^{new}_i (1)\) and\\
\(\hat{\tau}^{new}_0(\mathbf{x}) = \sum_{i:W_i^s=0} \alpha'_i(\mathbf{x}) Y^{new}_i(0)\); \\
\(\alpha_i(\mathbf{x})\) and \(\alpha'_i(\mathbf{x})\) are defined by Equations (\ref{alpha}), (\ref{alpha'}).

\begin{theorem}\label{theo:convtrans}
Let Assumptions \ref{hyp:vartrans} and \ref{hyp:numbertrans} be verified. Assume that for a fixed \(\beta > \frac{5}{2}\), \(C > 0\), each HTERF tree of the model \(M\) is the highest such that \(C \sqrt{n}(\ln n)^\beta \leq N_{n,0}(\mathbf{x}; \Theta_\ell, \mathcal{D}_n), N_{n,1}(\mathbf{x}; \Theta_\ell, \mathcal{D}_n)\). Assume that \(Y^s, Y^t\) and \(E^o, E^o_1, E^o_0\) error terms are bounded a.s. and that \(E^o, E^o_1, E^o_0\) converge to \(0\) in \(L^2\) as \(n_l\) tends to \(\infty\). Then,
\[
\mathbb{E}\left[ \left|\hat{\tau}^{new}_{B,n}(\mathbf{X}) - \tau^t(\mathbf{X}) \right|\right] \overset{}{\underset{n, n_l \rightarrow \infty}{\longrightarrow}} 0.
\]
\end{theorem}
\noindent The more technical parts of the proof are postponed to Appendix \ref{sec:prooftrans}.

\begin{remark}
With estimators \(\hat{f}^s_a\), \((a = 1, 0, \varnothing)\) of the form \(\sum \omega_i Y^s_i\), since \(Y\) is assumed to be bounded, so are \(\hat{f}^s_a\) and \(E^s_a\). \\
With \(\hat{f}^o_a\) of the form \(\sum \omega_i (Y^t_i - \hat{f}^s_a(\mathbf{X}_i^t))\), the error terms \(E^o_a\) are also bounded. Most classical regression algorithms provide estimators of this form: random forest, linear regression, neural network, etc.
\end{remark}

\begin{remark}
Following what is done in the proof of Theorem \ref{theo:convtrans}, the error on the estimation of \(\tau_1^t\) can be bounded as follows (the same rationale applies for \(\tau_0^t\)): let \(\mathbf{x} \in \mathcal{X}\),
\[
\left| \hat{\tau}_1^{new}(\mathbf{x}) - \tau_1^t(\mathbf{x}) \right| \leq B_{offset} + B_{HTERF}.
\]
Overall, this bound tends to \(0\) as \(n, n_l \rightarrow +\infty\). The second term is the bound of HTERF on a sample of size \(n\), and the first term is introduced by the offset method; the rate of convergence of this quantity only depends on the rate of convergence of \(\hat{f}^o\). More details of these bounds are presented in Appendix \ref{sec:bound}.
\end{remark}

\begin{proof}[Proof of Theorem \ref{theo:convtrans}.]
This proof is partially inspired by the proof of Theorem 4.1 in \cite{jocteur2024heterogeneous}. We provide details in the case of Algorithm \ref{alg:offsetcausal}; the proof for Algorithm \ref{alg:offsetcausalbis} follows the same lines by using \(\hat{f}^s(\mathbf{X}, W)\) and \(\hat{f}^o(\mathbf{X}, W)\) instead of \(\hat{f}^s_1(\mathbf{X})\), \(\hat{f}^s_0(\mathbf{X})\), \(\hat{f}^o_1(\mathbf{X})\), \(\hat{f}^o_0(\mathbf{X})\). We emphasize the details induced by the offset method in a causal framework and go more rapidly on parts that could be directly copied from \cite{jocteur2024heterogeneous}. \\
Let us consider an auxiliary dataset \(\mathcal{D}^{\diamond}_n = (Y^\diamond_i, \mathbf{X}^\diamond_i, W^\diamond_i)_{i=1, \dots, n}\) that is a sample of \((Y^t, \mathbf{X}^t, W^t)\), independent of \(\mathcal{D}^t\) and \(\mathcal{D}^s\). This new sample is used to build \((Y^{new, \diamond}_i)_{i=1, \dots, n}\), distributed as \((Y^{new}_i)_{i=1, \dots, n}\): 
if \(W^\diamond_i = 1\),
$$
Y_i^{new, \diamond} (1)= f_1^s(\mathbf{X}_i^\diamond) +f_1^o(\mathbf{X}_i^\diamond)  + \varepsilon_{1,i}^t +E_1^o(\mathcal{D}, \mathbf{X}_i^\diamond)
$$
with $ \varepsilon_{1, i}^{t}$  independent copies of $\varepsilon_1^t$, with the corresponding expression for $Y_i^{new, \diamond} (0)$ if $W_i^\diamond=0$.\\ 
Then,  the HTERF trees are grown using \(\mathcal{D}_n\), but the sample \(\mathcal{D}_n^{\diamond}\) (independent of \(\mathcal{D}_n\) and \(\Theta\)) is used to define a dummy estimator
\begin{align*}
&\tau^{new, \diamond}(\mathbf{x}; \Theta_1, \ldots, \Theta_k, \mathcal{D}_n^{\diamond}, \mathcal{D}_n) \\
=& \sum_{j=1, W_j^\diamond=1}^n \alpha_{n, j}^{\diamond}(\mathbf{x}; \Theta_1, \ldots, \Theta_k, \mathbf{X}^\diamond_1, \ldots, \mathbf{X}^\diamond_n,\\
&W^\diamond_1, \ldots, W^\diamond_n, \mathcal{D}_n) Y^{\diamond, new}_j (1)\\
&- \sum_{j=1, W_j^\diamond=0}^n \alpha_{n, j}^{'\diamond}(\mathbf{x}; \Theta_1, \ldots, \Theta_k, \mathbf{X}^\diamond_1, \ldots, \mathbf{X}^\diamond_n,\\
&W^\diamond_1, \ldots, W^\diamond_n, \mathcal{D}_n) Y^{\diamond, new}_j(0),
\end{align*}
where the weights are, for any \(j = 1, \ldots, n\),
\begin{align*}
&\alpha_{n, j}^{\diamond}(\mathbf{x}; \Theta_1, \ldots, \Theta_k, \mathbf{X}^\diamond_1, \ldots, \mathbf{X}^\diamond_n, W^\diamond_1, \ldots, W^\diamond_n, \mathcal{D}_n) \\
=& \dfrac{1}{B} \sum_{\ell=1}^B \dfrac{\indd_{\left\lbrace \mathbf{X}^\diamond_j \in A_n(\mathbf{x}; \Theta_\ell, \mathcal{D}_n) \right\rbrace \cap W^\diamond_j = 1 \land i \in \mathcal{D}_{n,2}^{\star}}}{N_{n, 1}^{\diamond}(\mathbf{x}; \Theta_\ell, \mathbf{X}^\diamond_1, \ldots, \mathbf{X}^\diamond_n, W^\diamond_1, \ldots, W^\diamond_n, \mathcal{D}_n)}.
\end{align*}
with \(N_{n, 1}^{\diamond}(\mathbf{x}; \Theta_\ell, \mathbf{X}^\diamond_1, \ldots, \mathbf{X}^\diamond_n, W^\diamond_1, \ldots, W^\diamond_n, \mathcal{D}_n)\), the number of elements of \(\mathcal{D}_n^{\diamond}\) that fall into \(A_n(\mathbf{x}; \Theta_\ell, \mathcal{D}_n)\) such that \(W^\diamond = 1\). Throughout this section, we shall use the convention \(\frac{0}{0} = 0\) in case
\(N_{n, 1}^{\diamond}(\mathbf{x}; \Theta_\ell, \mathbf{X}^\diamond_1, \ldots, \mathbf{X}^\diamond_n, W^\diamond_1, \ldots, W^\diamond_n, \mathcal{D}_n) = 0\) and thus \(\indd_{\left\lbrace \mathbf{X}^\diamond_j \in A_n(\mathbf{x}; \Theta_\ell, \mathcal{D}_n) \right\rbrace \cap W^\diamond_j = 1} = 0\) for \(j = 1, \ldots, n\).
Similarly, we have:
\begin{align*}
\begin{split}
&\alpha_{n, j}^{'\diamond}(\mathbf{x}; \Theta_1, \ldots, \Theta_k, \mathbf{X}^\diamond_1, \ldots, \mathbf{X}^\diamond_n, W^\diamond_1, \ldots, W^\diamond_n, \mathcal{D}_n) \\
=& \dfrac{1}{B} \sum_{\ell=1}^B \dfrac{\indd_{\left\lbrace \mathbf{X}^\diamond_j \in A_n(\mathbf{x}; \Theta_\ell, \mathcal{D}_n) \right\rbrace \cap W^\diamond_j = 0\land i \in \mathcal{D}_{n,2}^{\star}}}{N_{n, 0}^{\diamond}(\mathbf{x}; \Theta_\ell, \mathbf{X}^\diamond_1, \ldots, \mathbf{X}^\diamond_n, W^\diamond_1, \ldots, W^\diamond_n, \mathcal{D}_n)}.
\end{split}
\end{align*}
To lighten the notation in the sequel, we will simply write:
\begin{eqnarray*}
\tau_{B, n}^{new, \diamond}(\mathbf{x})& = &\sum_{j=1, W_j^\diamond=1}^n \alpha_j^{\diamond}(\mathbf{x}) Y^{\diamond, new}_j (1)- \sum_{j=1, W_j^\diamond=0}^n \alpha_j^{'\diamond}(\mathbf{x}) Y^{\diamond, new}_j(0) \\
&=& \tau^{new, \diamond}_1(\mathbf{x}) - \tau^{new, \diamond}_0(\mathbf{x})\/. 
\end{eqnarray*}
In other words, the $\mathcal{D}^\diamond_n$ sample allows  the forest machine and the estimation process to be disconnected in the analysis of the algorithm.\\
\ \\
Let \(\mathbf{x} \in \mathcal{X}\), we have:
\begin{align*}
\left| \hat{\tau}^{new}(\mathbf{x}) - \tau^t(\mathbf{x}) \right| \leq &\left| \hat{\tau}^{new}(\mathbf{x}) - \tau^{new, \diamond}(\mathbf{x}) \right| \\
&+ \left| \tau^{new, \diamond}(\mathbf{x}) - \tau^t(\mathbf{x}) \right|,
\end{align*}
and
\begin{align*}
\left| \tau^{new, \diamond}(\mathbf{x}) - \tau^t(\mathbf{x}) \right| \leq& \left| \tau^{new, \diamond}_1(\mathbf{x}) - \tau^t_1(\mathbf{x}) \right| \\
&+ \left| \tau^{new, \diamond}_0(\mathbf{x}) - \tau^t_0(\mathbf{x}) \right|.
\end{align*}
Remark that 
$$\tau_1^t(\mathbf{x}) = \mathbb{E}(Y^t(1)|\mathbf{X}^t=\mathbf{x}) = \mathbb{E}(Y^\diamond(1)|\mathbf{X}^\diamond=\mathbf{x})\/.$$ 
Because of the unconfoundedness hypothesis, it also rewrties:
$$\tau_1^t(\mathbf{x}) = \mathbb{E}(Y^t(1)|\mathbf{X}^t=\mathbf{x}\/, W^t=1) = \mathbb{E}(Y^\diamond(1)|\mathbf{X}^\diamond=\mathbf{x}\/, W^\diamond=1)\/.$$
Each of the two terms will be treated the same way.
\begin{gather*}
\left| \tau^{new, \diamond}_1(\mathbf{x}) - \tau^t_1(\mathbf{x}) \right| \leq \\
\left| \sum_{\substack{i=1 \\ W^{\diamond}_i=1}}^n \alpha^{\diamond}_i(\mathbf{x}) \left[ \left(Y^{new, \diamond}_i\right) - \mathbb{E}[Y^\diamond(1)|\mathbf{X}^\diamond=\mathbf{X}^{\diamond}_i] \right] \right| \\
+ \left| \sum_{\substack{i=1 \\ W^{\diamond}_i=1}}^n \alpha^{\diamond}_i(\mathbf{x}) \left[ \mathbb{E}[Y^t(1)|\mathbf{X}^t=\mathbf{X}^{\diamond}_i] - \mathbb{E}[Y^t(1)|\mathbf{X}^t = \mathbf{x}] \right] \right| \\
=: U_n + V_n.
\end{gather*}
The convergence of \(U_n\) and \(V_n\) to \(0\) in \(L^1\) is addressed in Appendix~\ref{sec:prooftrans}.
\end{proof}
\section{Simulation results}\label{sec:simul}
In the following examples, causal offset with unique and distinct models are compared to the baseline case where HTERF is simply trained on the available target data. Two choices of algorithms are considered to estimate functions \(f^s\) and \(f^o\), namely Kernel Ridge Regression and Regression Random Forest.
\subsection{One-dimensional example}
First, we consider a simple one-dimensional example where the source domain is defined by:\\
\(\mathbf{X}^s \sim U([0, 1])\), \(W^s \sim \text{Bern}(0.5)\), and\\
\(Y^s = \sin(\mathbf{X}^s) \cdot W^s + \cos(\mathbf{X}^s)\).\\
The target domain is defined by:\\
\(\mathbf{X}^t \sim U([0, 1])\), \(W^t \sim \text{Bern}(0.5)\), and\\
\(Y^t = \cos(\mathbf{X}^t) \cdot W^t + \cos(\mathbf{X}^t)\).\\
\ \\
A first simulation is performed with a source sample of size \(10{,}000\) and target samples of size \(500\). The results are recorded in Table~\ref{table:ex1D}.
{\small \begin{table}[h!]
\centering
\begin{tabular}{||c||c|cccc||}
\hline
Method    & RMSE   \\ \hline \hline
HTERF on source       &    0.003   \\ \hline
Offset KRR separate models    &      0.015     \\
Offset KRR unique model & 0.009      \\
No transfer (HTERF on target only) &      0.205      \\ \hline
\end{tabular}
\caption{\small One-dimensional example. Root mean squared errors of CATE on source and on target with three different methods: offset causal with separate KRR models, offset causal with unique KRR model, and HTERF only trained on target data (baseline method). HTERF causal forests have 500 trees. The results are aggregated over 50 simulation replications with 500 test points each (the source dataset stays unchanged; only the target training and test datasets are modified).}\label{table:ex1D}
\end{table}
}
\ \\
Table~\ref{table:ex1D} shows that both causal offset methods have better performance than the baseline method. In this example, using a single KRR model for treated and untreated individuals is more efficient than using two KRR models.\\
 \begin{figure}[H]
\centering
        \begin{subfigure}{.34\textwidth}
            \centering
            \includegraphics[width=.8\linewidth]{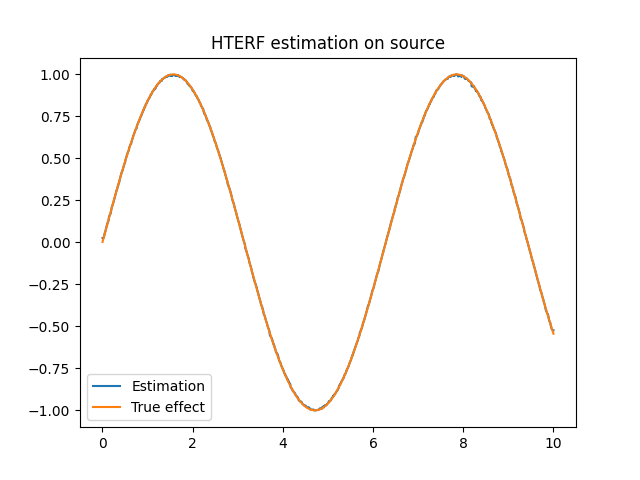}
            \label{fig:krrsource}
        \end{subfigure}
        \begin{subfigure}{.34\textwidth}
            \centering
            \includegraphics[width=.8\linewidth]{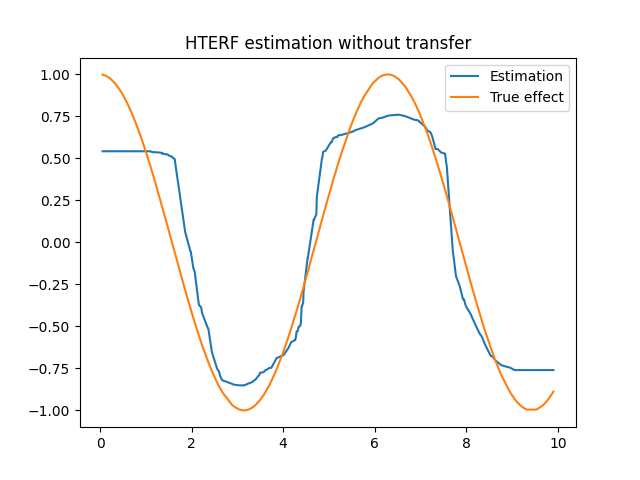}
            \label{fig:krrnotrans}
        \end{subfigure}
         \begin{subfigure}{.34\textwidth}
            \centering
            \includegraphics[width=.8\linewidth]{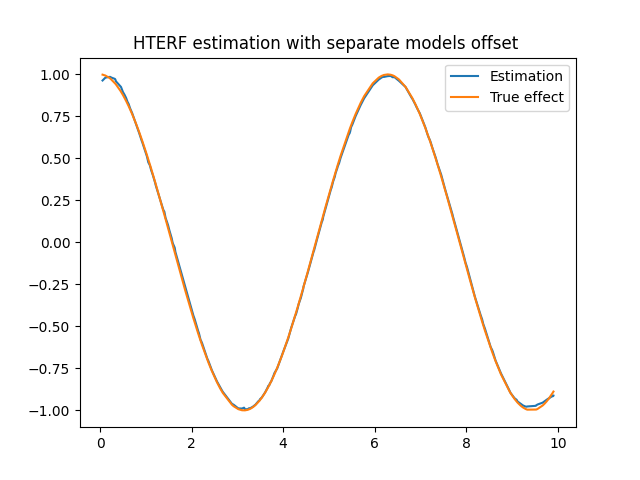}
            \label{fig:krrsepoff}
        \end{subfigure}
        \begin{subfigure}{.34\textwidth}
            \centering
            \includegraphics[width=.8\linewidth]{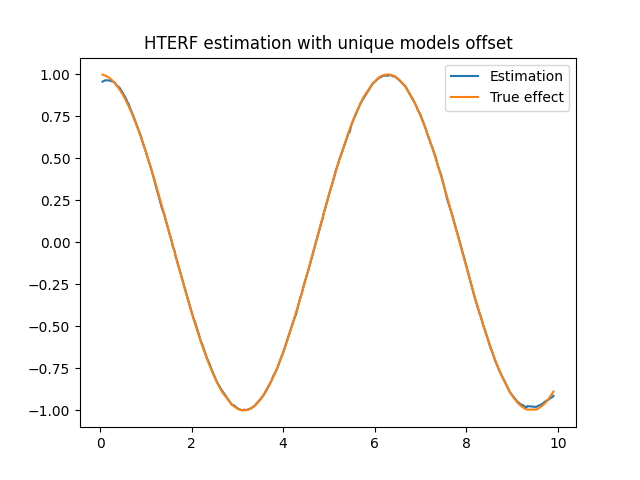}
            \label{fig:krruniqoff}
        \end{subfigure}
        \caption{\small Graphical illustration for the one-dimensional example. First: HTERF CATE estimation on source. Second: HTERF CATE estimation on target using only target data. Third: HTERF CATE estimation on target using offset causal with separate KRR models. Fourth: HTERF CATE estimation on target using offset causal with unique KRR model.}\label{krr1}
\end{figure}
\ \\ 
In order to explore the sensitivity of the method to various parameters (leaf size, number of trees, sample size), we conducted a more extensive study. In order to help reproducibility  of our work, the {\tt Python} code used to get these results  is available at \url{https://plmlab.math.cnrs.fr/maume/transfer_causal/}. We keep the same source and target models as above. The offset is done with regression random forests. The graphs below (Figure \ref{fig_dim1_leaf}) show the impact of leaf size, while Table \ref{table_dim1_leaf} gives more insights on other parameters. Following \cite{hill2011, hahn2020bayesian}, we use PEHE (\textit{precision in estimating heterogeneous effects}):  mean squared error of CATE estimates for each unit in a dataset, as an error measure.
\begin{figure}[H]
\includegraphics[scale=0.4]{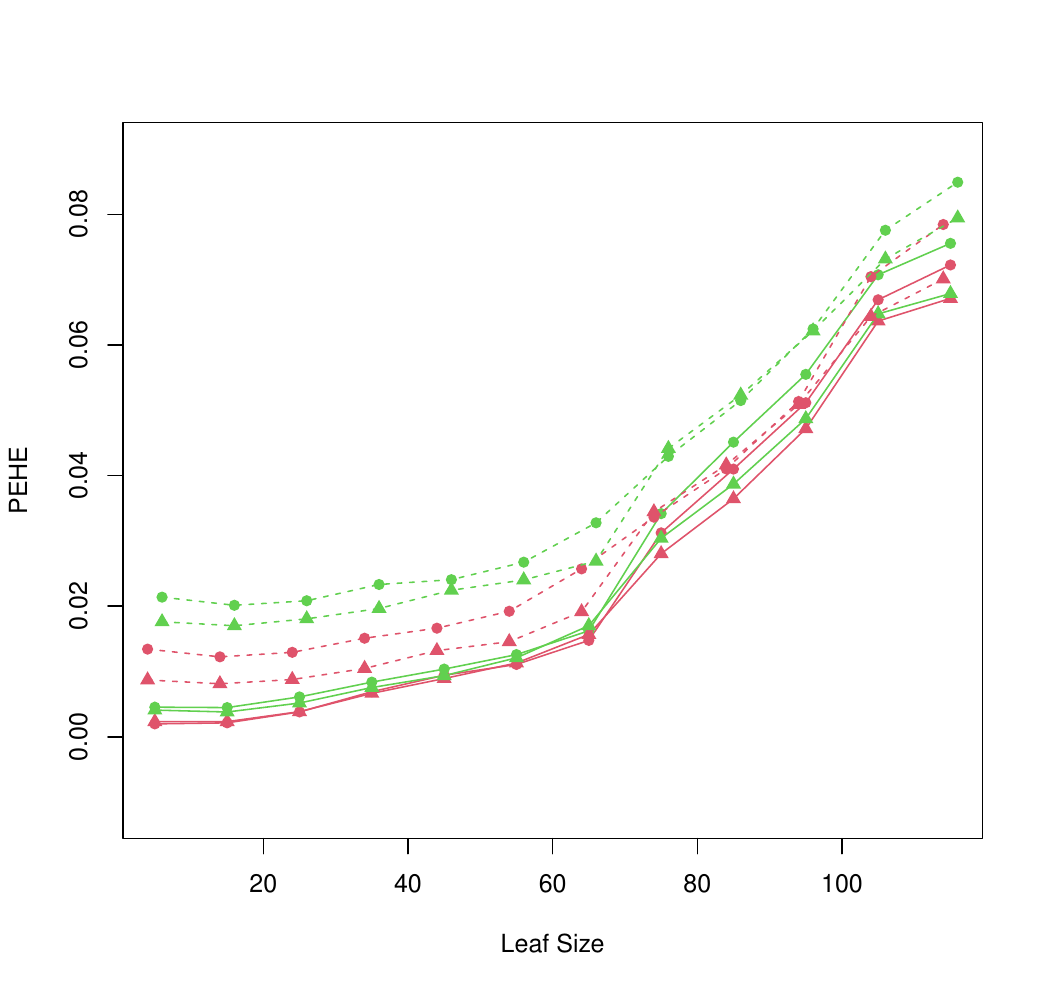}
\caption{Simple example in dimension 1. PEHE are computed over 200 test points and the mean is taken over 30 repetitions; the source sample size is 5,000. The red lines are for two offset models; the green lines are for one single offset model. Solid lines correspond to a ratio of target sample size over source sample size of 5\%, and the dotted lines correspond to a ratio of 2.5\%. The lines with round markers are for 500 trees, while the lines with triangles are for 200 trees.}\label{fig_dim1_leaf}
\end{figure}
{\small 
\begin{table}
\centering
\begin{adjustbox}{angle=90}
\begin{tabular}{|c|c||c|c|c||c|c|c|}
  \hline
   \multicolumn{2}{|c|}{ }& \multicolumn{3}{|c|}{Target/Source = 2,5\%} &  \multicolumn{3}{|c|}{Target/Source = 5\%}\\
Source (leaf) & Tree & \multicolumn{1}{c}{Double} & \multicolumn{1}{c}{Single} & \multicolumn{1}{c|}{Naive} & \multicolumn{1}{|c}{Double} & \multicolumn{1}{c}{Single} & \multicolumn{1}{c|}{Naive} \\
 \hline
 \multirow{5}{*}{ 10 000 (25) } & 100  & 0.249 (0.164) & 0.397 (0.338) & 51.391 (0.798)  & 0.062 (0.013) & 0.073 (0.030)  & 40.460 (3.062) \\
 & 200  &  0.201 (0.075) & 0.265 (0.115) & 53.478 (0.818) &  0.071 (0.019) & 0.071 (0.031) & 41.581 (4.32) \\
  & 300  & 0.234 (0.138) & 0.394 (0.269) & 53.101 (0.478) & 0.068 (0.017) & 0.113 (0.018) & 43.253 (1.903)\\
 & 400  &  0.208 (0.115) & 0.318 (0.189) & 54.198 (0.366) & 0.070 (0.017) & 0.077 (0.018) & 40.974 (3.450)\\
 & 500  &  0.209 (0.154) & 0.332 (0.268) & 52.942 (0.716) & 0.061 (0.016) & 0.067 (0.021) & 40.543 (2.839)\\
 \hline
 \hline 
  \multirow{5}{*}{ 5 000 (15) } & 100   & 1.359 (1.936) & 2.427 (1.951) & 55.855 (1.423) & 0.178 (0.047) & 0.288 (0.129) & 51.248 (1.398)  \\
  & 200  & 0.812 (0.499) & 1.699 (0.908) & 58.252 (2.654) &  0.232 (0.209) & 0.379 (0.272) & 49.022 (2.017)\\
  & 300  & 1.323 (2.179) & 2.562 (2.508) & 53.016 (2.651) & 0.203 (0.106) & 0.339 (0.205) & 49.299 (1.631)\\
  & 400  & 0.705 (0.547) & 1.623 (0.827) & 53.869 (1.389) & 0.168 (0.066) & 0.232 (0.099) & 47.646 (1.048)\\
  & 500  & 1.224 (2.378)  & 2.013 (2.418) & 52.781 (2.061) & 0.213 (0.120) & 0.447 (0.326) & 48.280 (2.172)\\
   \hline
 \hline 
  \multirow{5}{*}{ 2 000 (5) } & 100  & 6.02 (5.288) & 16.094 (9.559) & 56.484 (4.075) & 1.477 (1.509) & 2.516 (1.538) & 41.795 (5.199) \\
   & 200  & 8.431 (6.310) & 17.413 (10.330) & 55.721 (7.077) & 1.935 (2.383) & 4.228 (4.213) & 47.943 (2.709) \\
    & 300  & 8.377 (8.396) & 17.535 (9.812) & 55.632 (4.859) & 1.863 (2.051) & 3.664 (2.454) & 48.548 (4.996) \\
    & 400  & 7.173 (5.630) & 16.788 (8.099) & 55.108 (7.145) & 1.323 (1.396) & 3.321 (3.118) & 44.357 (3.714) \\
    & 500  &  6.567 (4.890) & 18.326 (8.704) & 55.767 (4.563) &2.336 (2.861) &3.598 (2.646) & 49.795 (7.523)\\
 \hline
\end{tabular}
\end{adjustbox}
\caption{\small For the simple dimension 1 example, we have recorded the mean over 30 repetitions of PEHE computed over 200 test points; in brackets is the standard deviation. The computations have been done for different sample sizes (source and target), number of trees, and the leaf size is in brackets. The ``Double'' columns refer to the method with two offset models; the ``Single'' columns refer to the model with one offset model; the ``Naive'' columns refer to the HTERF algorithm on the target sample (i.e., with no transfer). Means and standard deviations are multiplied by 100.}\label{table_dim1_leaf}
\end{table}
}
\ \\
Figure \ref{fig_dim1_leaf} and Table \ref{table_dim1_leaf} illustrate the convergence results with respect to the sample sizes, the influence of the leaf sizes, and the number of trees. Some cross-validation procedure could be used to select the best parameters, but it has not been implemented at this time. We remark on this one-dimensional example that the two offset models estimated by random forests give better results than the single random forest offset models. 
\subsection{Multi-dimensional example}\label{sec:transmulti}
Let us consider a multi-dimensional example inspired by the previous one. The source domain is:\\
\(\mathbf{X}^s = (\mathbf{X}^{s(1)}, \ldots, \mathbf{X}^{s(10)}) \sim U([0, 1]^{10})\),\\
\(W^s \sim \text{Bern}(0.5)\), and\\
\(Y^s = \sin(\mathbf{X}^{s(1)}) W^s + \cos(\mathbf{X}^{s(2)})\). The target domain is defined by: \(\mathbf{X}^t \sim U([0, 1]^{10})\),\\
\(W^t \sim \text{Bern}(0.5)\), and\\
\(Y^t = \cos(\mathbf{X}^{t(1)}) W^t + \cos(\mathbf{X}^{t(2)})\).\\
\ \\
A first simulation is done with a source sample of size 10,000 and target samples of size 500.
{\small 
\begin{table}[h!]
\centering
\begin{tabular}{||c||c|cccc||}
\hline
Method    & RMSE   \\ \hline \hline
Source       &    0.004   \\ \hline
Offset separate KRR models    &      0.957   \\
Offset unique KRR model  & 0.960      \\ \hline
Offset separate RF models    &      0.120   \\
Offset unique RF model & 0.135      \\ \hline
No transfer &      0.348      \\ \hline
\end{tabular}
\caption{\small Multi-dimensional example. Root mean squared errors of CATE on source and on target with five different methods: offset causal with separate KRR models, offset causal with unique KRR model, offset causal with separate random forest (RF) models, offset causal with unique RF model, and HTERF only trained on target data (baseline method). HTERF causal forests have 500 trees. The results are aggregated over 50 simulation replications with 500 test points each (the source dataset stays unchanged; only the target training and test datasets are modified).}\label{table:exmultD}
\end{table}
}
\ \\
Figure~\ref{krr2} illustrates the poor performance of KRR in estimating \(f^s\) and \(f^o\). In the top right image, the KRR estimator of the function \(f^s\) fails to capture the variations of \(Y^s_0\) against \(\mathbf{X}^{s(2)}\). However, when using random forests to estimate \(f^s\) and \(f^o\), causal offset algorithms are more efficient than the baseline method (see Table \ref{table:exmultD}). In this setting, using separate models for treated and control units is the best strategy.
\begin{figure}[H]
\centering
        \begin{subfigure}{.34\textwidth}
            \centering
            \includegraphics[width=.8\linewidth]{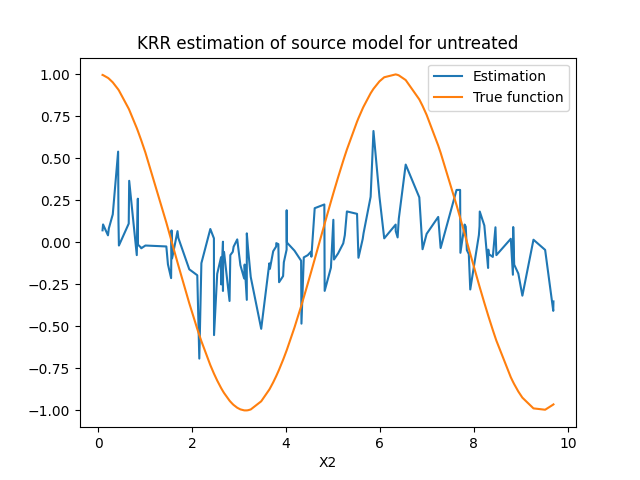}
            \label{fig:krrsource2}
        \end{subfigure}
        \begin{subfigure}{.34\textwidth}
            \centering
            \includegraphics[width=.8\linewidth]{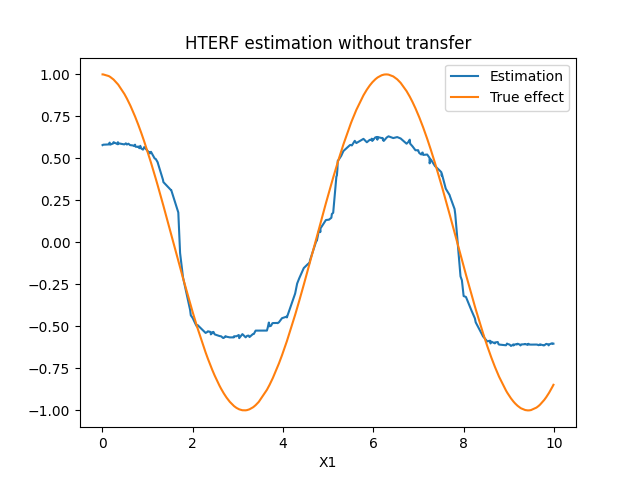}
            \label{fig:krrnotrans2}
        \end{subfigure}
         \begin{subfigure}{.34\textwidth}
            \centering
            \includegraphics[width=.8\linewidth]{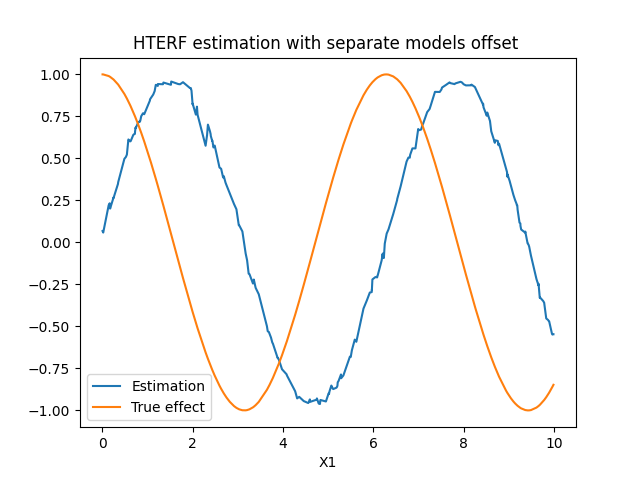}
            \label{fig:krrsepoff2}
        \end{subfigure}
        \begin{subfigure}{.34\textwidth}
            \centering
            \includegraphics[width=.8\linewidth]{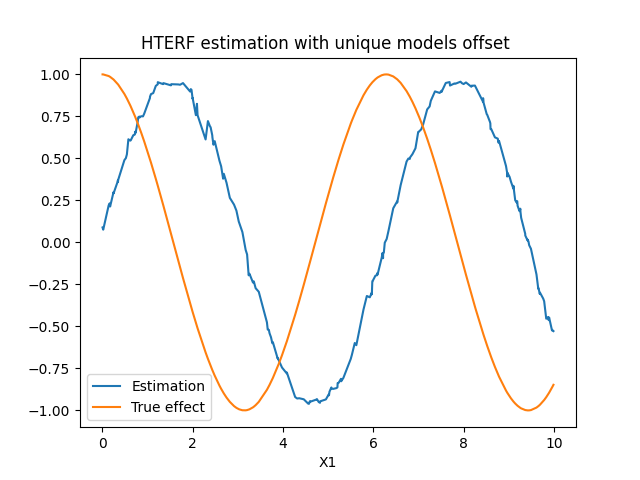}
            \label{fig:krruniqoff2}
        \end{subfigure}
        \caption{\small Graphical illustration for the multi-dimensional example. First: KRR estimation on source of the function \(f_0^s\). Second: HTERF CATE estimation on target using only target data. Third: HTERF CATE estimation on target using offset causal with separate KRR models. Fourth: HTERF CATE estimation on target using offset causal with unique KRR model.}\label{krr2}
\end{figure}
\ \\
As in the one-dimensional case, we conducted a more extensive study to assess the impact of sample sizes, leaf sizes, and the number of trees on the results. The performance measure used is again PEHE. The offset is calculated using regression random forests. Figure \ref{fig_dimsup_leaf} illustrates the impact of leaf size, while Table \ref{table_dimsup_leaf} provides further insights into the other parameters.\\
\begin{figure}[H]
\includegraphics[scale=0.4]{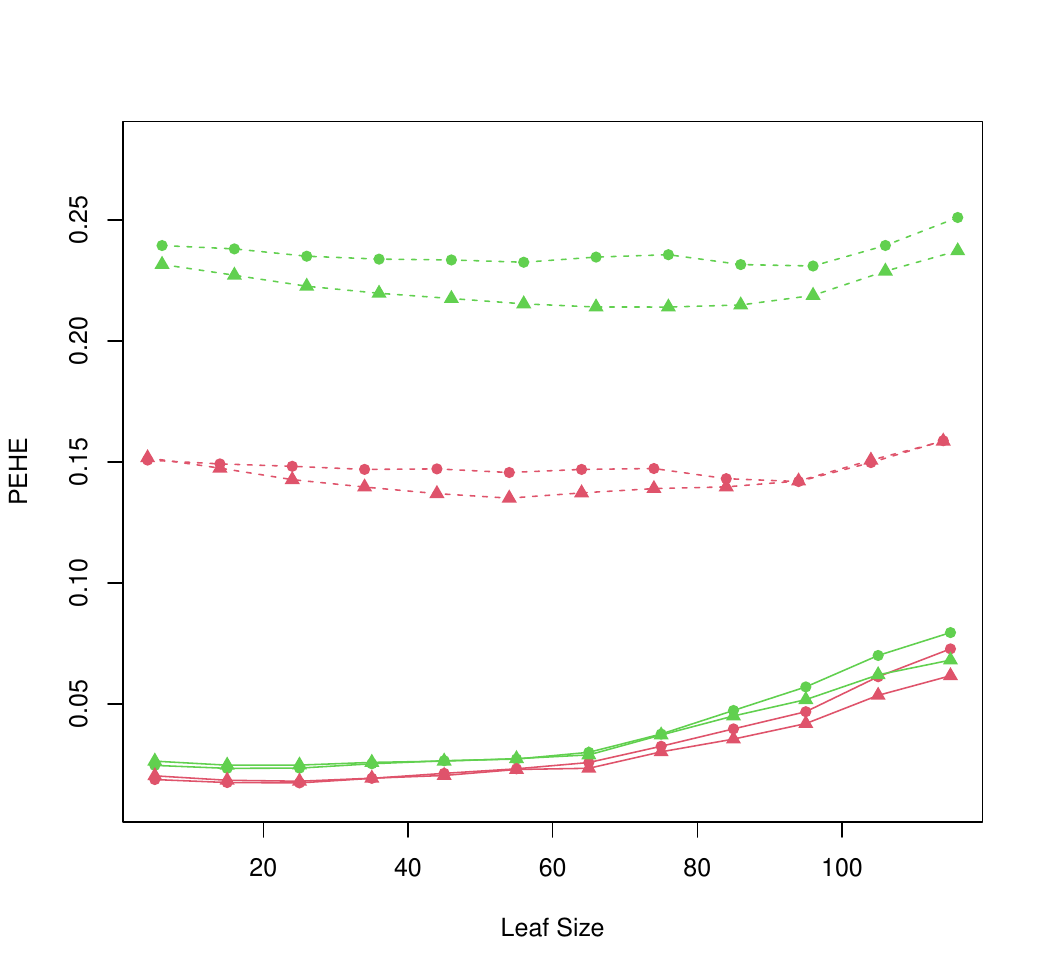}
\caption{\small Example in higher dimension. PEHE are computed over 200 test points and the mean is taken over 30 repetitions; the source sample size is 5,000. The red lines are for two offset models; the green lines are for one single offset model. Solid lines correspond to a ratio of target sample size over source sample size of 5\%, and the dotted lines correspond to a ratio of 2.5\%. The lines with round markers are for 500 trees, while the lines with triangles are for 200 trees.}\label{fig_dimsup_leaf}
\end{figure}
{\small
\begin{table}[h!]
\centering
\begin{adjustbox}{angle=90}
\begin{tabular}{|c|c||c|c|c||c|c|c|}
  \hline
   \multicolumn{2}{|c|}{ }& \multicolumn{3}{|c|}{Target/Source = 2,5\%} &  \multicolumn{3}{|c|}{Target/Source = 5\%}\\
Source & Tree (leaf)& \multicolumn{1}{c}{Double} & \multicolumn{1}{c}{Single} & \multicolumn{1}{c|}{Naive} & \multicolumn{1}{|c}{Double} & \multicolumn{1}{c}{Single} & \multicolumn{1}{c|}{Naive} \\
 \hline
 \multirow{5}{*}{ 10 000 (45) } & 100  & 1.637 (0.962) & 2.061 (1.233) & 56.753 (0.809) & 0.461 (0.084) & 0.514 (0.126) & 55.237 (0.184) \\
 & 200  & 1.552 (0.462) & 1.882 (0.723) & 54.814 (0.305) & 0.434 (0.088) & 0.471 (0.097) & 53.792 (0.683)\\
 & 300  & 1.722 (0.545) & 2.240 (0.755) & 51.135 (0.810) & 0.428 (0.068) & 0.476 (0.069) & 51.218 (0.519) \\
 & 400  & 1.299 (0.406) & 1.719 (0.487) & 52.315 (0.700) & 0.415 (0.087) & 0.451 (0.107) & 52.901 (0.643)\\
 & 500 & 1.635 (0.616) & 2.251 (0.676) & 53.689 (0.796) & 0.046 (0.070) & 0.496 (0.097) & 52.982 (0.415)\\
 \hline
 \hline 
  \multirow{5}{*}{ 5 000 (35) } & 100  &  17.379 (6.779) & 27.096 (10.853) & 53.626 (2.108) & 1.866 (0.505) & 2.568 (0.819) & 50.090 (1.399)\\
  & 200  &  13.967 (10.830) & 21.973 (11.071) & 52.598 (1.627) & 1.938 (0.889) & 2.591 (0.933) &  53.476 (0.662) \\
  & 300  & 13.713 (8.329) & 20.674 (7.285) & 59.665 (4.051) & 2.074 (0.557) & 3.013 (0.955) & 53.985 (0.702)\\
  & 400  & 18.109 (12.446) & 26.480 (11.622) & 52.722 (1.908) & 2.111 (0.671) & 2.977 (0.597) & 50.223 (0.979)\\
  & 500  & 14.697 (8.907) & 23.383 (10.371) & 51.288 (2.606) & 1.934 (0.462) & 2.535 (0.542) & 53.296 (0.906)\\
   \hline
 \hline 
  \multirow{5}{*}{ 2 000 (25) } & 100  & 66.089 (15.218) & 77.807 (11.887) &  56.293 (9.582) & 24.284 (11.014) & 39.270 (10.134) & 51.797 (2.947) \\
  & 200 & 65.477 (12.587) & 77.681 (7.363) & 57.284 (9.536) & 20.093 (9.438) & 33.099 (9.842) & 54.067 (3.345)\\
  & 300 & 56.928 (13.383) & 68.575 (8.217) & 51.632 (4.537) & 26.992 (15.559) & 38.181 (11.114) & 53.646 (3.100)\\
  & 400 & 70.315 (17.692) & 79.726 (8.729) & 60.383 (7.795) & 25.630 (15.875) & 35.515 (11.915) & 57.631 (3.967)\\
  & 500 & 53.397 (11.787) & 67.583 (9.483) & 54.560 (7.482) & 26.673 (11.469) & 36.297  (9.735) & 51.801 (3.876)\\
 \hline
\end{tabular}
\end{adjustbox}
\caption{\small Example in higher dimension. We have recorded the mean over 30 repetitions of PEHE computed over 200 test points; in brackets is the standard deviation. The computations have been done for different sample sizes (source and target), number of trees, and the leaf size is in brackets. Means and standard deviations are multiplied by 100.}\label{table_dimsup_leaf}
\end{table}
}
\noindent Remark that in this dimension 10 model, only the first two variables are used in the model. The eight remaining variables may be seen as noise. Thus, this example shows the robustness of our algorithms to nuisance variables. It should also be noted that, in this example, for source sample sizes and target sample sizes that are too small, applying the transfer method yields worse results than the naive method (without transfer), see Table \ref{table_dimsup_leaf}.
\subsection{A more complex model: Ishigami-like}
We conclude this simulation study with a more complex and non-linear model: an Ishigami-like model in dimension 3. The source domain is:\\
\(\mathbf{X}^s = (\mathbf{X}^{s(1)},\mathbf{X}^{s(2)} , \mathbf{X}^{s(3)}) \sim U([0, 1]^3)\), \\
\(W^s \sim \text{Bern}(0.5)\), and
\begin{eqnarray*}
 Y^s &=& \sin(\mathbf{X}^{s(1)}) + \left[\sin(\mathbf{X}^{s(2)})^2\right] \cdot W^s\\
 &&+ 0.1 [\mathbf{X}^{s(3)}]^2 \cdot \sin(\mathbf{X}^{s(1)}). 
\end{eqnarray*}
The target domain is defined by:\\
\(\mathbf{X}^t \sim U([0, 1]^3)\), \(W^t \sim \text{Bern}(0.5)\), and 
\begin{eqnarray*}
 Y^t &= &\sin(\mathbf{X}^{t(1)}) + \left[\cos(\mathbf{X}^{t(2)})^2\right] \cdot W^t\\
 &&+ 0.1 [\mathbf{X}^{t(3)}]^2 \cdot \sin(\mathbf{X}^{t(1)})\/.
\end{eqnarray*}
We conducted a study illustrating the impact of sample sizes, leaf sizes, and number of trees on the results. The performance measure is again PEHE. The offset is done with regression random forests. Figure \ref{fig_ishi_leaf} shows the impact of leaf size; Table \ref{table_ishi_leaf} gives more insights on other parameters.
\begin{figure}[h!]
\includegraphics[scale=0.4]{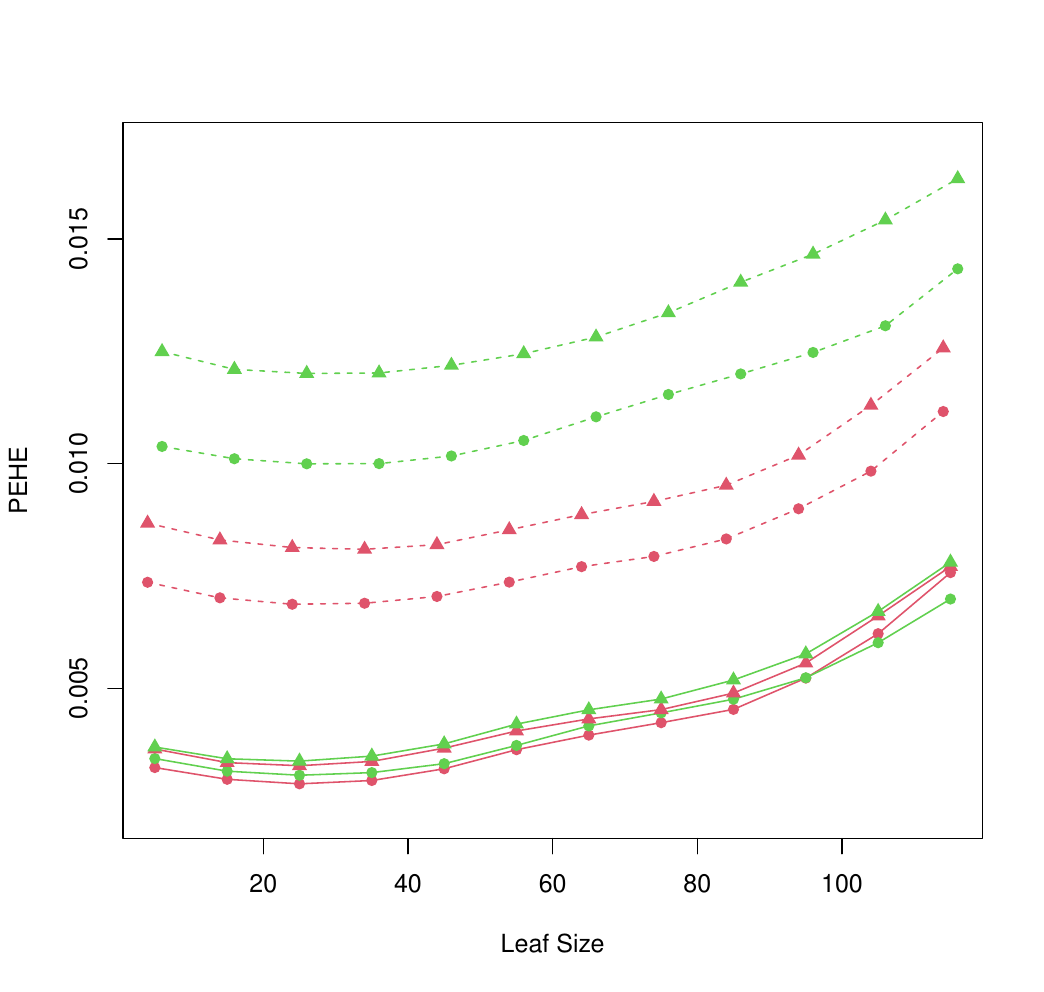}
\caption{\small An Ishigami-like example. PEHE are computed over 200 test points and the mean is taken over 30 repetitions; the source sample size is 10,000. The red lines are for two offset models; the green lines are for one single offset model. Solid lines correspond to a ratio of target sample size over source sample size of 5\%, and the dotted lines correspond to a ratio of 2.5\%. The lines with round markers are for 300 trees, while the lines with triangles are for 100 trees.}\label{fig_ishi_leaf}
\end{figure}
{\small
\begin{table}[h!]
\centering
\begin{adjustbox}{angle=90}
\begin{tabular}{|c|c||c|c|c||c|c|c|}
  \hline
   \multicolumn{2}{|c|}{ }& \multicolumn{3}{|c|}{Target/Source = 2,5\%} &  \multicolumn{3}{|c|}{Target/Source = 5\%}\\
Source & Tree (leaf)& \multicolumn{1}{c}{Double} & \multicolumn{1}{c}{Single} & \multicolumn{1}{c|}{Naive} & \multicolumn{1}{|c}{Double} & \multicolumn{1}{c}{Single} & \multicolumn{1}{c|}{Naive} \\
 \hline
 \multirow{3}{*}{ 10 000 (35) } & 100  & 0.809 (0.333) & 1.202 (0.619) & 12.996 (0.184) & 0.337 (0.069) & 0.349 (0.080) & 12.565 (0.141)   \\
 & 300  &  0.689 (0.173) & 1.000 (0.277) & 11.915 (0.344) & 0.295 (0.051) & 0.313 (0.077) & 11.356 (0.330) \\
 & 500 &   0.797 (0.244) & 1.113 (0.437) & 13.106 (0.254) & 0.321 (0.060) & 0.345 (0.059) & 11.480 (0.225) \\
 \hline
 \hline 
  \multirow{3}{*}{ 5 000 (25) } & 100  &  2.648 (1.975) & 3.843 (2.226) & 12.275 (0.236) & 1.264 (0.281) & 1.836 (0.689) & 13.628 (0.417)\\
  & 300  & 2.745 (1.210) & 4.256 (1.585) & 13.006 (0.671) & 1.219 (0.355) & 1.823 (0.971) & 12.062 (0.316)\\
  & 500  & 2.211 (0.764) & 3.978 (1.369) & 13.257 (0.222) & 1.075 (0.392) & 1.565 (0.691) & 12.113 (0.163)\\
   \hline
\end{tabular}
\end{adjustbox}
\caption{\small An Ishigami-like model. We have recorded the mean over 30 repetitions of PEHE computed over 200 test points; in brackets is the standard deviation. The computations have been done for different sample sizes (source and target), number of trees, and the leaf size is in brackets. Means and standard deviations are multiplied by 100.}\label{table_ishi_leaf}
\end{table}
}
\noindent For this Ishigami-like example, since it is more complex, we have done simulations for source sample sizes of 5,000 and 10,000; the target sizes are 2.5\% or 5\%. In this case, the results without transfer are very poor, which emphasizes the interest of the transfer method.
\subsection{Semi-synthetic dataset}
A final example is presented using the IHDP dataset, already introduced in the HTERF article \cite{jocteur2024heterogeneous}. This dataset has been studied in \cite{wei2024transfer} to compare the accuracy of various ATE estimators in a transfer learning context. In addition to RMSE, two additional indicators of performance for ATE estimation are added:
\begin{itemize}
    \item ATE1: \(\frac{1}{n_u} \left| \sum_{i=1}^{n_u} \hat{\tau}^{new}_{B,n}(\mathbf{X}_i^{tU}) - \tau^t(\mathbf{X}_i^{tU}) \right|\)
    \item ATE2: \(\frac{1}{n_u} \sum_{i=1}^{n_u} \left| \hat{\tau}^{new}_{B,n}(\mathbf{X}_i^{tU}) - \tau^t(\mathbf{X}_i^{tU}) \right|\)
\end{itemize}
To create the source and target domains, the binary variable ``The mother drank alcohol during pregnancy'' has been used: the source domain consists of all children whose mothers did not drink, and the target domain consists of the children whose mothers drank alcohol.
{\small
\begin{table}[h!]
\centering
\begin{tabular}{||c||ccc||}
\hline
Method    & CATE & ATE1 & ATE2 \\ \hline \hline
Offset separate models    &  0.808  & 0.292 &   0.561   \\
Offset unique model & 0.734 & 0.180 &   0.491    \\
No transfer & 1.016   & 0.351 &  0.732   \\ \hline
\end{tabular}
\caption{\small RMSE on CATE and two different errors on ATE with three different methods: offset causal with separate RF models, offset causal with unique RF model, and HTERF only trained on target data (baseline method). HTERF causal forests have 500 trees; the forest of the first step in HTERF has 500 trees. The results are aggregated over 50 simulation replications; the source dataset stays unchanged, but for each replication, the target sample is modified.}\label{table:exihdp}
\end{table}
}
\ \\
In this example, causal offset with a unique random forest model is the most efficient for CATE and ATE estimation, followed by causal offset with two separate models (see Table \ref{table:exihdp}). Both algorithms outperform the baseline method without transfer.
\ \\
\section{Discussion}\label{sec:conclusion}
In this work, we have presented an algorithm to perform transfer learning on the causal inference problem. This approach combines the offset algorithm, already used on regression problems, and the HTERF causal forest. The combination of these two methods allows for a consistency result on the CATE estimation in the target domain. A generalization bound is also shown; these results rely on stronger assumptions than the classical HTERF consistency, especially regarding the number of trees in the forest.
\\
\ \\
Additional work could be done on the proof of consistency to lighten the assumptions. An almost sure convergence might also be obtained instead of an \(L^1\) convergence. Also, from a practical point of view, a cross-validation procedure would be useful, especially to select the optimal leaf size, even if the results are quite good by taking any leaf size in a window from around 5 to 50. One issue in order to propose a cross-validation procedure is that the quantity of interest (CATE) is not observed, as it would be the case in regression. Some work would have to be done to find the appropriate criterion for cross-validation.
\appendix %
\section{Proof of results} \label{sec:prooftrans}
\begin{proof}[Sequel of the proof of Theorem \ref{theo:convtrans}.]
The \(U_n\) term rewrites:
\begin{align*}
    U_n =& \left|\sum_{\substack{i=1\\W^{ \diamond}_i=1}}^n  \alpha^{\diamond}_i (\mathbf{x}) \left(\varepsilon^{t}_{1,i} -  E_1^{o}(\mathcal{D},\mathbf{X}_i^\diamond)\right)\right|\\
    \leq& \left|\sum_{\substack{i=1\\W^{ \diamond}_i=1}}^n  \alpha^{ \diamond}_i \varepsilon^{t}_{1,i} \right| +\left|\sum_{\substack{i=1\\W^{ \diamond}_i=1}}^n  \alpha^{\diamond}_i E_1^{o}(\mathcal{D},\mathbf{X}_i^\diamond)\right|.
\end{align*}
Since \(Y^s, Y^t\) are assumed to be bounded, \(f^s_\cdot, f^o_\cdot\) are assumed to be continuous, and \(\mathbf{X}\) lives in a compact space, then \(\varepsilon^s_\cdot\) and \(\varepsilon^t_\cdot\) are bounded. Following the HTERF consistency proof (see the Appendix of \cite{jocteur2024heterogeneous} and the use of Lemma A.7 there), we have (where $C$ denotes a generic positive constant):
\begin{eqnarray*}
\mathbb{E} \left[\left|\sum_{\substack{i=1\\W^{ \diamond}_i=1}}^n  \alpha^{ \diamond}_i \varepsilon^{t}_{1,i} \right|^2\right] &\leq&  \Vert\varepsilon_1^{t}\Vert_\infty^2 \cdot \left(\frac{C}{ \sqrt{n}(\ln n)^\beta}\right.\\
\lefteqn{\hspace*{-2cm} \left.+ \frac{C^2}{n(\ln n)^\gamma} + C \sqrt{n}(n+1)^{2d}e^{- C(\ln n)^{2\beta}/C}\right)}\\
    &&\rightarrow 0\ \mbox{as} \ n\rightarrow \infty\/.
\end{eqnarray*}
For the last term, 
we use the following decomposition:
\begin{eqnarray*}
\mathbb{E} \left[  \alpha^{\diamond}_i \left|E_1^{o}(\mathcal{D},\mathbf{X}_i^\diamond)\right| \right] &\leq& \mathbb{E} \left[  \left|\alpha^{\diamond}_i -\frac{1}{n}\right|\left|E_1^{o}(\mathcal{D},\mathbf{X}_i^\diamond)\right| \right] \\
&&+ \frac{1}{n} \mathbb{E}\left[  \left|E_1^{o}(\mathcal{D},\mathbf{X}_i^\diamond)\right|\right]
\end{eqnarray*}
Since \(\sum_{j=1}^n \alpha_j^\diamond = 1\) and the \(\mathbf{\alpha}_j^\diamond\)'s are identically distributed conditionally on $\mathcal{D}$, we have for any $n\in \mathbb{N}$:  \(\mathbb{E}[\alpha_i^\diamond|\mathcal{D}] = \frac{1}{n}\) and \(\mathbb{E}[\alpha_i^\diamond] = \frac{1}{n}\). The Cauchy-Schwarz inequality gives:
\begin{eqnarray*}
   \lefteqn{ \mathbb{E} \left[  \left|\alpha^{\diamond}_i -\frac{1}{n}\right|\left|E_1^{o}(\mathcal{D}\/,\mathbf{X}_i^\diamond)\right| \right] }\\
   &&\leq \sqrt{Var(\alpha^{\diamond}_i)}\sqrt{\mathbb{E} \left[  \left|E_1^{o}(\mathcal{D},\mathbf{X}_i^\diamond)\right|^2\right]}.
\end{eqnarray*}
Using the total variance formula:
\[
Var(\alpha^{\diamond}_i) = \mathbb{E} \left[ Var(\alpha^{\diamond}_i|\mathcal{D})  \right].
\]
We can rewrite:
\[
\alpha_i^\diamond = \frac{1}{B} \sum_{l=1}^B \frac{\indd_{\mathbf{X}^\diamond_i \in A_n(l)}}{N^\diamond_{n,1}(\mathbf{x};\Theta_l,\mathcal{D}_n)} =: \frac{1}{B} \sum_{l=1}^B Z_l,
\]
where conditionally on \(\mathcal{D}\), the \((Z_l)_{l \in \{1, \dots, n\}}\) are independent and identically distributed; this leads to
\begin{align*}
Var(\alpha^{\diamond}_i|\mathcal{D}) &= \frac{1}{B} Var(Z_1) \\
&\leq  \frac{1}{B} \mathbb{E}[Z_1^2|\mathcal{D}] \\
&\leq \frac{1}{B} \mathbb{E} \left[\frac{\indd_{\mathbf{X}^\diamond_i \in A_n(1)}}{(N^\diamond_{n,1}(\mathbf{x};\Theta_1,\mathcal{D}_n))^2} \Big| \mathcal{D}\right] \\
&\leq \frac{1}{B} \mathbb{E} \left[\frac{\indd_{\mathbf{X}^\diamond_i \in A_n(1)} \indd_{\{N^\diamond_{n,1} \left( \mathbf{x};\Theta_1,\mathcal{D}_n \right) \geq \lambda\}}}{( N^\diamond_{n,1}(\mathbf{x};\Theta_1,\mathcal{D}_n))^2} \Big| \mathcal{D}\right] \\
&+ \frac{1}{B} \mathbb{E} \left[\frac{\indd_{\mathbf{X}^\diamond_i \in A_n(1)} \indd_{\{N^\diamond_{n,1} \left( \mathbf{x};\Theta_1,\mathcal{D}_n \right) < \lambda\}}}{( N^\diamond_{n,1}(\mathbf{x};\Theta_1,\mathcal{D}_n))^2} \Big| \mathcal{D}\right].
\end{align*}
Let \(\lambda = \frac{\sqrt{n}(\ln n)^\beta}{2}\),
\begin{align*}
Var(\alpha^{\diamond}_1|\mathcal{D}) &\leq \frac{1}{B\lambda} \mathbb{E} \left[\frac{\indd_{\mathbf{X}^\diamond_i \in A_n(1)}}{ N^\diamond_{n,1}(\mathbf{x};\Theta_1,\mathcal{D}_n)} \Big| \mathcal{D}\right] \\
&+ \frac{1}{B} \mathbb{P} \left(N^\diamond_{n,1} \left( \mathbf{x};\Theta_1,\mathcal{D}_n \right) < \lambda \Big| \mathcal{D}\right) \\
&\leq \frac{1}{B\lambda} \mathbb{E} \left[\alpha^{\diamond}_i|\mathcal{D}\right] \\
&+ \frac{1}{B} \mathbb{P} \left( N^\diamond_{n,1} \left( \mathbf{x};\Theta_1,\mathcal{D}_n \right) < \lambda \Big| \mathcal{D}\right) \\
&\leq \frac{1}{B\lambda n} + \mathbb{P} \left( N^\diamond_{n,1} \left( \mathbf{x};\Theta_1,\mathcal{D}_n \right) < \lambda \Big| \mathcal{D}\right).
\end{align*}
Remark that
\begin{equation*}
\begin{multlined}
    \left\{ N^\diamond_{n,1} \left( \mathbf{x};\Theta_1,\mathcal{D}_n \right) < \frac{\sqrt{n}\left(\ln n\right)^\beta}{2} \right\} \\ \subset \left\{ \left|N_{n,1} \left( \mathbf{x};\Theta_1,\mathcal{D}_n \right) - N^\diamond_{n,1} \left( \mathbf{x};\Theta_1,\mathcal{D}_n \right)\right| \right. \\
    > \left. \frac{\sqrt{n}\left(\ln n\right)^\beta}{2} \right\},
    \end{multlined}
\end{equation*}
thus we have
\begin{equation*}
\begin{multlined}
    \mathbb{P} \left( N^\diamond_{n,1} \left( \mathbf{x};\Theta_1,\mathcal{D}_n \right) < \lambda \Big| \mathcal{D}\right) \leq \\ \mathbb{P} \left( \left|N_{n,1} \left( \mathbf{x};\Theta_1,\mathcal{D}_n \right) - N^\diamond_{n,1} \left( \mathbf{x};\Theta_1,\mathcal{D}_n \right)\right| > \lambda \Big| \mathcal{D}\right).
\end{multlined}
\end{equation*}
The following lemma has been stated in \cite{jocteur2024heterogeneous}, it is Lemma A.7 there; it is based on Vapnik-Chervonenkis theory.
\begin{lemma} \label{lem:proxdiamtrans}
    Consider \(u \in \{0,1\}\), as before, \(N_{n,u}(A_n(\Theta)) = N_{n,u}(\mathbf{x};\Theta, \mathcal{D}_n)\) is the number of observations of \(\mathcal{D}_n\) such that \(W=u\) that fall into \(A_n(\Theta) = A_n(\mathbf{x};\Theta,\mathcal{D}_n)\), and \(N_{n,u}^{\diamond}(A_n(\Theta)) = N_{n,u}^{\diamond}(\mathbf{x};\Theta, \mathbf{X}^{\diamond 1}, \ldots, \mathbf{X}^{\diamond n}, \mathcal{D}_n)\), the number of observations of \(\mathcal{D}_n^{\diamond}\) such that \(W=u\) that fall into \(A_n(\Theta)\). Then, \(\forall \varepsilon > 0\),
 \begin{align*}
 \quad \mathbb P \left( \left| N_{n,u} \left( A_n \left( \Theta \right) \right) - N_{n,u}^{\diamond} \left( A_n \left( \Theta \right) \right) \right| > \varepsilon \right) \\
 \leqslant 16 (n+1)^{2d}e^{- \varepsilon^2/128n} \ .
\end{align*}
\end{lemma}
\noindent Using Assumption \ref{hyp:numbertrans} and Lemma \ref{lem:proxdiamtrans}, there exist \(C\) and \(M\) positive constants such that:
\begin{align*}
Var(\alpha^{\diamond}_i) \leq& \frac{1}{B\lambda n} +\mathbb E\left[\mathbb P \left( \left|N_{n,1} \left( \mathbf x;\Theta_1,\mathcal{D}_n \right)- \right. \right. \right.\\
&\left.\left.\left. N^\diamond_{n,1} \left( \mathbf x;\Theta_1,\mathcal{D}_n \right)\right|>\frac{\sqrt n \left( \ln n\right)^\beta}{2}\Big|\mathcal{D}\right)\right]\\
\leq& \frac{2}{B n^{3/2}\left(\ln n\right)^\beta} +\mathbb P \left( \left|N_{n,1} \left( \mathbf x;\Theta_1,\mathcal{D}_n \right)-\right.\right.\\
&\left.\left. N^\diamond_{n,1} \left( \mathbf x;\Theta_1,\mathcal{D}_n \right)\right|>\frac{\sqrt n \left( \ln n\right)^\beta}{2}\right)\\
    \leq & \frac{2}{M n^2 } + 4 C (n+1)^{2d}e^{- (\ln n)^{2\beta}/512}\\
 &   =\mathcal{O}\left(\frac{1}{n^2}\right). 
\end{align*}
Now, since we assume that \(E_1^o\) converges to \(0\) in \(L^2\) as $n$ goes to $\infty$, it also converges to $0$ in \(L^1\):
\begin{align*}
   \mathbb E \left[ \left|\sum_{\substack{i=1\\W^{ \diamond}_i=1}}^n  \alpha^{\diamond}_i E_1^{o}(\mathcal{D},\mathbf X_i^\diamond)\right| \right] &\leq \sum_{\substack{i=1\\W^{ \diamond}_i=1}}^n\sqrt{Var(\alpha^{\diamond}_i)} \cdot \\
    &\sqrt{\mathbb E \left[  \left|E_1^{o}(\mathcal{D},\mathbf X_i^\diamond)\right|^2\right]} \\
    &+ \mathbb E\left[  \left|E_1^{o}(\mathcal{D},\mathbf X_1^\diamond)\right|\right] \cdot \\
  &  \leq   n\sqrt{Var(\alpha^{\diamond}_1)}\\&\sqrt{\mathbb E \left[  \left|E_1^{o}(\mathcal{D},\mathbf X_1^\diamond)\right|^2\right]}\\
  &+ \mathbb E\left[  \left|E_1^{o}(\mathcal{D},\mathbf X_1^\diamond)\right|\right]\\
    &\longrightarrow 0\ \mbox{as} \ n \rightarrow \infty\/.
\end{align*}
The term \(V_n\) can now be treated:
\begin{align*}
 V_n& = \left | \sum_{\substack{i=1\\W^{ \diamond}_i=1}}^n \alpha^{ \diamond}_i(\mathbf x) \left[   \mathbb E (Y^t(1)|\mathbf{X}^t=\mathbf X^{\diamond}_i ) \right.\right. \\
 &\left.\left.- \mathbb E (Y^t(1)|\mathbf X^t=\mathbf x )  \right] \right|
\end{align*}
The following lemma allows to control the variation of $\tau_1$ on leaves. It is Lemma 2 in \cite{meinshausen2006quantile} and it is similar to Lemma 5 in \cite{benard2022mean}.
\begin{lemma}\label{lem:diameter}
    Let Assumptions \ref{hyp:vartrans} and \ref{hyp:numbertrans} be verified, let \(\mathbf{x} \in \mathcal{X}\) and \(\ell \in [1, B]\). Denote \(A_n(\mathbf{x}, \Theta_\ell, \mathcal{D}_n) = \bigotimes_{j=1}^d I(\mathbf{x}, \Theta_\ell, \mathcal{D}_n)\), where \(I(\mathbf{x}, \Theta_\ell, \mathcal{D}_n)\) are intervals; then,
    \[
    \max_{j=1, \dots, d} |I(\mathbf{x}, \Theta_\ell, \mathcal{D}_n)| = o(1).
    \]
\end{lemma}
\noindent Combining Lemma \ref{lem:diameter} with the continuity of \(\tau_1\), we get
\[
\forall \ell \in [1, B], \forall \mathbf{x} \in \mathcal{X}, \sup_{\mathbf{z} \in A_n(\mathbf{x}, \Theta_\ell, \mathcal{D}_n)} |\tau_1(\mathbf{z}) - \tau_1(\mathbf{x})| \overset{a.s.}{\underset{n \rightarrow \infty}{\longrightarrow}} 0.
\]
Using this result, we get that $V_n$ goes  to \(0\) almost surely, as $n$ goes to infinity: indeed, if  for all $\ell =1\/, \ldots \/, B$ $ \mathbf{X}^\diamond_i \not\in A_n(\mathbf{x}, \Theta_\ell, \mathcal{D}_n)$ then $\alpha^\diamond_i(\mathbf{x}) =0$ so that 
\begin{align*}
         \left| \sum_{\substack{i=1\\W^\diamond_i=1}}^n \alpha^\diamond_i(\mathbf{x}) \left[   \mathbb{E} \left[ Y^t(1)|\mathbf{X}^t=\mathbf{X}^\diamond_i \right]  - \mathbb{E} \left[ Y^{t}(1)|\mathbf{X}^t=\mathbf{x} \right] \right] \right| \\
        \leq  \sup_{\mathbf{z} \in A_n(\mathbf{x})} \left| \tau_1^t(\mathbf{z}) - \tau_1^t(\mathbf{x}) \right| \xrightarrow[n \rightarrow +\infty]{} 0\/.
\end{align*}
Since  \(\tau_1\) is bounded, so is $V_n$ and by the dominated convergence theorem, we have the \(L^1\) convergence to $0$ of $V_n$, when $n$ goes to infinity.
%
\\
\ \\
It remains to consider the quantity \(\left| \hat{\tau}^{new}(\mathbf{x}) - \tau^{new, \diamond}(\mathbf{x}) \right|\). We  consider separately but in a similar fashion \(\left| \hat{\tau}^{new}_1(\mathbf{x}) - \tau^{new, \diamond}_1(\mathbf{x}) \right|\) and \(\left| \hat{\tau}^{new}_0(\mathbf{x}) - \tau^{new, \diamond}_0(\mathbf{x}) \right|\):
\begin{eqnarray*}
 \lefteqn{\spadesuit= \left|    \hat{\tau}^{new}_1(\mathbf x) - \tau^{new, \diamond}_1(\mathbf x) \right|=} \\
 && \left | \frac{1}{B}  \sum_{l=1}^B\sum_{j=1}^n \frac{\indd_{\mathbf X_j\in A_n(l)}\indd_{W_j^s=1}}{N_{n,1}(\mathbf x;\Theta_l,\mathcal{D}_n)} Y_j^{new}(1) -\right.\\
 &   &\left.\frac{\indd_{\mathbf X^\diamond_j\in A_n(l)}\indd_{W^\diamond_j=1}}{N^\diamond_{n,1}(\mathbf x;\Theta_l,\mathcal{D}_n)} Y^{new, \diamond}_j(1) \right |.
\end{eqnarray*}
Following the HTERF consistency proof in \cite{jocteur2024heterogeneous} (see the end of the Appendix there, where the similar quantity is considered), $\spadesuit$ converges to \(0\) almost surely, as $n$ goes to infinity. The crutial point here is that up to noise ($\varepsilon^t$) and error ($E^0_1$) terms,  $Y^{new,\diamond}(1)$ and $Y^{new}(1)$ are distributed as $f_1^s(\mathbf{X}^s)+f_1^o(\mathcal{X}^s)$. \\  
As already mentionned, since all  \(Y^t\) is bounded and $f^s$, $f^o$ are continuous, we have that $Y^{new,\diamond}$ and $Y^{new}$ are bounded. So that the dominated convergence theorem, implies that the $\spadesuit$ term tends to \(0\) in \(L^1\) as $n$ goes to infinity.
\end{proof}

\section{Generalisation bound}\label{sec:bound}

We start with the following decomposition:
\begin{align*}
\lefteqn{\left| \hat{\tau}_1^{new}(\mathbf{x}) - \tau_1^t(\mathbf{x}) \right|}\\
\leq &\left| \hat{\tau}_1^{new}(\mathbf{x}) -\tau_1^{new, \diamond}(\mathbf{x}) \right| +  \underbrace{\left| \tau^{new, \diamond}_1(\mathbf{x}) - \tau^t_1(\mathbf{x}) \right|}_{\leq U_n+V_n} \\
\leq& \left| \frac{1}{B}  \sum_{l=1}^B \sum_{j=1}^n \frac{\indd_{\mathbf{X}_j \in A_n(l)} \indd_{W_j=1}}{N_{n,1}(\mathbf{x};\Theta_l,\mathcal{D}_n)} Y_j^{new} \right. \\
&\left.- \frac{\indd_{\mathbf{X}^\diamond_j \in A_n(l)} \indd_{W^\diamond_j=1}}{N^\diamond_{n,1}(\mathbf{x};\Theta_l,\mathcal{D}_n)} Y^{new, \diamond}_j \right|  +\left| \sum_{\substack{i=1\\W^{ \diamond}_i=1}}^n  \alpha^{\diamond}_i E_1^{o}(\mathcal{D},\mathbf{X}_i^\diamond)\right| + \\
& \left|\sum_{\substack{i=1\\W^{ \diamond}_i=1}}^n  \alpha^{ \diamond}_i \varepsilon^{t}_{1,i} \right| +\left| \sum_{\substack{i=1\\W^{\diamond}_i=1}}^n \alpha^{\diamond}_i(\mathbf{x}) \left[   \mathbb{E} \left[ Y^{t}(1)|\mathbf{X}^t=\mathbf{X}^{\diamond}_i \right]  - \right.\right.\\
&\left.\left.\mathbb{E} \left[ Y^{t}(1)|\mathbf{X}^t=\mathbf{x} \right]  \right] \right| \\
 \\
\leq& B_{offset} + B_{HTERF}\/,
\end{align*}
where
\begin{align*}
    B_{offset} =& \left|\sum_{\substack{i=1\\W^{ \diamond}_i=1}}^n  \alpha^{\diamond}_i (\mathbf{x}) E_1^{o}(\mathcal{D},\mathbf{X}_i^\diamond)\right| + \\
    & \left|\sum_{\substack{i=1\\W^{ \diamond}_i=1}}^n  \alpha^{ \diamond}_i \varepsilon^{t}_{1,i} \right|  \\
\end{align*}
and
\begin{align*}
     B_{HTERF} =& \left| \sum_{\substack{i=1\\W^{\diamond}_i=1}}^n \alpha^{\diamond}_i(\mathbf{x}) \left[   \mathbb{E} \left[ Y^{t}(1)|\mathbf{X}=\mathbf{X}^{\diamond}_i \right]\right.\right.  -\\
     &\left.\left.\mathbb{E} \left[ Y^{t}(1)|\mathbf{X}^t=\mathbf{x} \right]  \right] \right|+ \\
& \left| \frac{1}{B}  \sum_{l=1}^B \sum_{j=1}^n \frac{\indd_{\mathbf{X}_j \in A_n(l)} \indd_{W_j=1}}{N_{n,1}(\mathbf{x};\Theta_l,\mathcal{D}_n)} Y_j^{new}\right. -\\
&\left.\frac{\indd_{\mathbf{X}^\diamond_j \in A_n(l)} \indd_{W^\diamond_j=1}}{N^\diamond_{n,1}(\mathbf{x};\Theta_l,\mathcal{D}_n)} Y^{new, \diamond}_j \right|.
\end{align*}

\section*{Acknowledgements} We are grateful to Natixis, specifically the Enterprise Risk Management department, for partially funding this work. 
\bibliographystyle{apalike}
\bibliography{sample}

\end{document}